\documentclass[11pt]{article}
\pdfoutput=1

\usepackage{ifthen}
\usepackage{bbm}

\usepackage[
  backend=biber,
  style=alphabetic,
  maxalphanames=12,
  maxcitenames=2,
  maxbibnames=99,
  giveninits=false,
  doi=false,
  url=false,
  natbib=true
]{biblatex}

\DeclareDelimFormat[bib]{nametitledelim}{\addspace}

\usepackage{bm}
\usepackage{fullpage}
\usepackage[english]{babel}
\usepackage[utf8]{inputenc}
\usepackage[T1]{fontenc}
\usepackage{ae} \usepackage{aecompl}
\usepackage{enumitem}
\setlist{nosep}
\usepackage{microtype}
\usepackage{booktabs}
\usepackage{hhline}
\usepackage{makecell}
\usepackage{etoolbox}
\apptocmd{\sloppy}{\hbadness 10000\relax}{}{}

\usepackage{amsmath,amsthm,amssymb,mathtools,thmtools}
\usepackage{graphicx}
\usepackage[textsize=tiny]{todonotes}
\usepackage[colorlinks=true, allcolors=blue]{hyperref}
\usepackage{algorithm,algorithmicx,algpseudocode}
\usepackage{xfrac}
\usepackage[normalem]{ulem}

\hypersetup{colorlinks,
            breaklinks=true,
            linkcolor=blue,
            citecolor=blue,
            urlcolor=magenta,
            linktocpage,
            plainpages=false,
            bookmarks=false}

\usepackage{notations}
\usepackage{comment}
\usepackage{tikz,pgfplots,pgfplotstable}
\usepgfplotslibrary{units}
\usepackage{tabularx} 
\usepackage{subfig}
\usepackage{array}
\usepgfplotslibrary{groupplots}

\title{Private Evaluation of High-Dimensional Functions}

\author{
Mohammad Afzali\thanks{University of British Columbia. {\tt moafzali@cs.ubc.ca}. Supported in part by Four Year Doctoral Fellowship (4YF) from the University of British Columbia.}
\and
Hassan Ashtiani\thanks{McMaster University and the Vector Institute. {\tt zokaeiam@mcmaster.ca}. Supported by an NSERC Discovery Grant.}
\and 
Adam Smith\thanks{Boston University. 
\texttt{ads22@bu.edu}. Supported in part by NSF award CNS-2232694 and an Apple faculty research award.}
}

\numberwithin{equation}{section}
\hypersetup{hypertexnames=false}
\begin{document}

\maketitle

\begin{abstract}%
We study differentially private black-box evaluation of arbitrary high-dimensional functions. Given a dataset $S$ and query access to a function $f$---whose sensitivity may be unknown or unbounded---the goal is to privately output an approximation to $f(S)$. 
The black-box evaluator can only query the value of $f$ over a down neighborhood of $S$ (i.e., those subsets of $S$ that are obtained by removing a bounded number of elements from $S$).
We design generic black-box evaluators for both pure and approximate differential privacy through a reduction to high-dimensional private quasi-concave optimization (QCO). 

Our black-box evaluator enjoys a geometric notion of utility. Consider the vectors obtained by applying $f$ on sets in a down neighborhood of $S$. With high probability, the evaluator's output is guaranteed to be in the convex hull of these vectors. 

Our work generalizes the result of \citet{linder2025privately} for real-valued function evaluation to high dimensions. It also recovers the private center-point guarantees of \citet{pmlr-v99-beimel19a}, with an explicit additive approximation to maximum Tukey depth.
We also apply our evaluators to user-level quasi-concave and convex optimization, showing how they can improve upon the sample-and-aggregate baseline of \citet{nissim2007smooth} by replacing the multiplicative data overhead with an additive overhead.

Finally, we show that our formulation of high dimensional black-box evaluation and high dimensional private QCO are equivalent up to a factor of the dimension: the hardest private evaluation instances already arise from the cases that $f$ is a (non-private) quasi-concave optimizer. 

\end{abstract}

\allowdisplaybreaks


\section{Introduction}
Differential privacy provides a mathematical framework for analyzing/processing sensitive data while limiting what the released output reveals about any one individual~\cite{dwork2006calibrating,dwork2014algorithmic}. It requires that changing an individual's data have only a limited effect on the output distribution of the data processing algorithm. A basic task within this framework is \emph{private function evaluation}: given a sensitive dataset $S$ in some input domain $\cU^n$ and a function $f$, we want to release an accurate approximation to $f(S)$ while satisfying differential privacy. Here, $f(S)$ may be a statistic, such as a mean, or a vector of parameters computed by a non-private optimization procedure.

A standard way to privately evaluate $f$ is to add noise calibrated to its \emph{global sensitivity}~\cite{dwork2006calibrating}. Consider a real-valued function $f:\cU^*\to\bR$, where $\cU^*$ denotes the set of finite datasets over $\cU$. The global sensitivity of $f$ is
$\Delta_f=\sup_{S\sim S'}|f(S)-f(S')|$,
where the supremum is over all neighboring pairs of datasets: those that differ by the addition or removal of a single record. This approach raises two distinct challenges.

First, a valid sensitivity bound may be unavailable or untrustworthy. A curator may have only black-box access to a function submitted by an untrusted analyst, and privacy cannot rely on the correctness of the analyst's claimed sensitivity bound. Even when the implementation is available, establishing a valid bound can be difficult for complex computations. This challenge concerns how to guarantee privacy without trusted sensitivity information, regardless of whether the function's actual global sensitivity is small or large.

Second, global sensitivity can be too large to yield useful accuracy, even on datasets where the function is stable. This motivates data-dependent notions of sensitivity. The \emph{local sensitivity} of $f$ at $S$ is
$\mathrm{LS}^f(S)=\sup_{S'\sim S}|f(S)-f(S')|$.
Although local sensitivity can be much smaller than global sensitivity, directly calibrating noise to it does not generally preserve privacy, because the noise scale itself can reveal information about the dataset. The smooth sensitivity framework of \citet{nissim2007smooth} addresses this problem through smoothly varying upper bounds on local sensitivity that support private noise calibration. However, both local and smooth sensitivity account for arbitrary insertions as well as deletions. Consequently, both quantities can
be large or infinite even when $f$ is stable under deletions from the observed dataset.

The two challenges are also connected computationally. With only black-box access to an arbitrary function, evaluating its behavior on all insertion neighbors requires exploring the input universe, which may be prohibitively large or infinite. Restricting attention to subsets of the observed dataset avoids this universe-dependent search and captures stability under deletions.

These considerations motivate a data-dependent measure of stability
that considers only deletions~\cite{fang2022shifted,kohli2023differential}. Following \citet{linder2025privately},
the \emph{down sensitivity} of a real-valued function $f$ at $S$
and depth $\lambda\geq 0$ is defined as
\begin{align*}
    \mathrm{DS}^{f}_{\lambda}(S)
    =\max_{S'\in N^{\downarrow}_{\lambda}(S)}
      |f(S)-f(S')|
\end{align*}
where $N^{\downarrow}_{\lambda}(S)
    =\{S'\subseteq S: |S\setminus S'|\leq\lambda\}$
denotes the sub-datasets obtained by deleting at most $\lambda$ records from $S$. Down sensitivity can be computed using finitely many evaluations of $f$ on subsets of $S$. Furthermore, it can be much smaller than smooth sensitivity.
For example, let $f$ be the function that outputs the mean of a dataset over the domain of real values. Then $f$ has infinite local sensitivity for every nonempty dataset: inserting a single (large) value can change the mean arbitrarily. Therefore, its smooth sensitivity is also infinite. Nevertheless, for $\lambda<|S|$, its down sensitivity satisfies $\mathrm{DS}^{f}_{\lambda}(S)
    \leq \frac{\lambda}{|S|-\lambda}
           (\max S-\min S)$,
which is finite and can be small when the observed data have a narrow range.

The above observations motivate the notion of a \emph{private black-box evaluator} \cite{linder2025privately} that aims to approximate $f(S)$ given dataset $S$ and query access to any real-valued $f$: the evaluator accesses $f$ only by querying its value on $S$ or its subsets, and must guarantee differential privacy with respect to $S$ for every $f$.
\citet{linder2025privately} construct such an evaluator
whose output lies between the minimum and maximum values
of $f$ on $N^{\downarrow}_{\lambda}(S)$---hence satisfies $|\hat y-f(S)|\leq\mathrm{DS}^{f}_{\lambda}(S)$---with high probability.

The depth $\lambda$ depends on the privacy parameters, the desired
success probability, and the output domain size\footnote{One should consider finite output domains $\cY\subseteq\bR$, since even the basic task of privately returning an interior point---a value between the minimum and maximum input values---requires $\Omega(\log^*|\cY|)$ samples and is impossible over the full real line~\cite{bun2015differentially}.}. Thus, the evaluator
guarantees privacy for every such function, while providing small
error whenever the function is stable under deletions from the
observed dataset.

Our goal is to address private black-box evaluation for \emph{high-dimensional} vector-valued functions. A simple baseline for achieving this is to apply the one-dimensional evaluator of \citet{linder2025privately} to each coordinate separately,
allocating the privacy budget through composition. The output of the resulting approach would be guaranteed to be in an axis-aligned box---whose side-lengths correspond to coordinate-wise down sensitivities---around the true value of the function. 

Such a box can lose the relationships between coordinates. For example, the function's outputs may vary substantially along a diagonal (that is, not axis-aligned) direction while remaining constant in a perpendicular direction. Bounds on the individual coordinates do not capture this directional stability:
\begin{example}
The \emph{Tukey depth}~\cite{tukey1975mathematics} of $x\in\bR^d$ with respect to a set of points $S$ is the minimum number of such points in any closed halfspace containing $x$. A Tukey median is a point with maximal Tukey depth.
Consider the set $S$ containing equally-many copies of $(0,0)$ and $(1,1)$, and let $f$ be the function that returns a Tukey median. Deleting one record can make either endpoint the unique median. However, the smallest axis-aligned rectangle containing the dataset is $[0,1]\times[0,1]$. Allowing outputs in that set permits $(0,1)$, whose Tukey depth is zero. In contrast, the points on the segment from $(0,0)$ to $(1,1)$ have optimal Tukey depth $n/2$. 
\end{example}

Examples like this motivate an accuracy guarantee based on the \emph{convex hull} of the output vectors. Specifically, we consider the following problem:
\begin{quote}
    Design a private black-box evaluator that given dataset $S$ and function $f$ taking values in $\bR^d$,  outputs, with high probability,
    $$\hat y\in\CH \Big(\big\{f(S'): S'\in N^{\downarrow}_{\lambda}(S)\big\}\Big)\, .$$
\end{quote}

where $\CH$ denotes the convex hull of a set. The goal is to achieve a small deletion depth $\lambda$ that depends on the dimension $d$, the privacy parameters $\eps,\delta$, with only mild dependence on $|\cY|$ and no dependence on the dataset size or the size of the input universe.
 
The convex-hull based guarantee respects variation in every direction, including directions that do not align with the coordinates.

This geometric property makes the evaluator useful for the optimization applications developed in this paper. In one dimension, the convex hull is an interval, recovering the one-dimensional notion of accuracy of \citet{linder2025privately}.

\subsection{Our contributions and techniques}

We construct evaluators for arbitrary vector-valued functions with a finite output domain $\cY\subseteq \bR^d$, under both pure and approximate differential privacy. Privacy holds for every black-box function $f$ without a supplied sensitivity bound. The utility guarantee is that, with high probability, the evaluator's output lies in $\CH(\{f(S'):S'\in N^\downarrow_\lambda(S)\})$, while making black-box calls $f(T)$ only for subdatasets $T\in N^\downarrow_\lambda(S)$ on every run.

\begin{theorem}[Informal version of \ref{thm:main-eval-approx} and \ref{thm:main-eval-pure}]
    Let $\cY\subseteq \bR^d$ be an output domain. Then there exist:
    \begin{enumerate}
        \item A $(\eps,\delta)$-differentially private black-box evaluator with $\lambda=\Tilde{O}\left(d^{2.5} 8^{\log^* |\cY|}\frac{\log(1/\delta)}{\eps}\right)$.
        \item A $(\eps,0)$-differentially private black-box evaluator with $\lambda=O\left(\frac{d^{3}\log|\cY|}{\eps}\right)$.
    \end{enumerate}
\end{theorem}
For real-valued functions, this recovers the one-dimensional evaluation guarantees of \citet{linder2025privately}, based on down sensitivity.%
\footnote{\citet{linder2025privately} also give results for the ``claimed sensitivity'' setting; see related work in Section~\ref{sec:related}.} Our construction of black-box evaluator reveals a deep connection with private quasi-concave optimization: 
\begin{theorem}[Informal version of~\ref{thm:main} and \ref{thm:main-pure}]
Private evaluation of arbitrary black-box functions can be reduced to \emph{private quasi-concave optimization (QCO)}.    
\end{theorem}

In private QCO~\cite{beimel2013private,cohen2023optimal}, we are given a dataset $S$, a finite domain of candidate vectors $\cV\subseteq\bR^d$, and a real-valued quality function $q$. For each candidate vector $y\in\cV$, $q(S,y)$ measures its utility on $S$. The goal is to privately return a vector in $\CH(\cV)$ with utility close to the maximum over the candidate domain. To solve QCO privately, the quality function must satisfy the following two properties:
\begin{itemize}
    \item For every fixed $y$, the function $q(.,y)$ has bounded global sensitivity in its first argument.
    \item For every dataset $S$, the function $q(S,.)$ is quasi-concave in its second argument.
\end{itemize}

Now, given an arbitrary black-box function $f$ with potentially unbounded sensitivity, the main challenge is to define a quality function $q(S,y)$ on points $y\in\bR^d$ that satisfies the two properties mentioned above. Furthermore, it should certify that $y\in\CH(\{f(S'):S'\in N^\downarrow_\lambda(S)\})$ whenever the quality score $q(S,y)$ is sufficiently high.

Our reduction constructs a depth function (Definition~\ref{def:depth}) that measures how deeply a candidate vector lies within the convex hull $\CH\Big(\big\{f(S'): S'\in N^{\downarrow}_{\lambda}(S)\big\}\Big)$ of the down-neighborhood image of $f$.
A single record can drastically change this image and its geometry, yet we design the depth function to have sensitivity at most $1$ and to be quasi-concave. At the same time, high depth certifies the desired convex-hull membership. Turning this construction into a private evaluator requires overcoming a further geometric challenge: proving that sufficiently deep points exist for every black-box function and ensuring that a finite candidate domain captures them. In Section~\ref{sec:black-box-eval-qco}, we show how to overcome these challenges to obtain a reduction from private black-box evaluation to private QCO.

Conversely, we can treat a non-private quasi-concave optimizer as the black-box function, and evaluate it using our private black-box evaluator. Combined with the previous theorem, this gives:
\begin{theorem}[Informal version of \ref{thm:black-box-qco-equi}]
    Up to a factor of dimension $d$, private black-box evaluation and private QCO are equivalent.    
\end{theorem}

In particular, this means that the hardest instances of black-box evaluation in this sense already arise from non-private optimizers of quasi-concave quality functions: solving private QCO suffices for universal evaluation, and a universal evaluator must also be able to solve private QCO.

Finally, we develop a generic high-dimensional QCO solver in Theorem~\ref{thm:high-dim-opt-1}. The one-dimensional solver of \citet{beimel2013private} takes a lower bound on the optimal quality as input and guarantees a multiplicative approximation to that bound. We develop a one-dimensional solver that follows the same core recursive approach and provides an agnostic additive-error guarantee: it approximates the actual optimum without requiring its value or a lower bound as input. We then extend this solver to $d$ dimensions using coordinatewise QCO techniques previously employed for private center-point computation and halfspace optimization~\cite{pmlr-v99-beimel19a,kaplan2020private,nissim2026differentially}. The resulting generic high-dimensional QCO solver may be of independent interest.

Our black-box evaluation framework provides a general tool for privately solving a broad range of optimization problems. We illustrate this with the following applications:

    \emph{Private center points.} We recover the result of \citet{pmlr-v99-beimel19a} for privately finding center points with an explicit additive approximation to the dataset's maximum Tukey depth in Corollary~\ref{lemma:tm-improved}.

  \emph{User-level private optimization.}
In Section~\ref{sec:black-box-opt}, we obtain user-level private algorithms for quasi-concave and convex optimization using only black-box access to non-private optimizers. Our evaluator requires no supplied sensitivity or Lipschitz bound and no uniform bound on users' contribution sizes. Utility depends on the worst non-private solution after removing a small number of users, evaluated on the original dataset. Our evaluator incurs an additive overhead in the required number of users, whereas the sample-and-aggregate baseline incurs a multiplicative overhead~\cite{nissim2007smooth}.

\subsection{Related work}
\label{sec:related}

Almost immediately after the introduction of the global sensitivity framework~\cite{dwork2006calibrating}, a line of work began that aims to get accuracy guarantees tailored to a particular input (e.g., \cite{nissim2007smooth,mcsherry2007mechanism}). The thread closest to our setting, launched by \cite{jha2013testing}, considers a setting in which the analyst supplies additional parameters describing a tentative (not promised) upper bound on the sensitivity of $f$ \cite{awasthi2014limitations,cummings2020individual,kohli2023differential,lange2025local}.

Following \citet{linder2025privately}, we call this the \emph{claimed sensitivity bound setting}. The private evaluators in this setting guarantee differential privacy even when the supplied sensitivity bound is incorrect; however, their stated accuracy guarantees apply only when the bound is valid.

On the other hand, when no claimed sensitivity bound is available, \citet{fang2022shifted} developed a private evaluator for monotone functions with a public finite output domain.\footnote{A function is \emph{monotone} if adding records to the dataset cannot decrease its value.} Their privacy guarantee relies on monotonicity. \citet{linder2025privately} extended this approach to arbitrary real-valued functions, providing a private evaluator that requires no sensitivity bound and guarantees privacy without assuming monotonicity. 

More recently, \citet{pmlr-v336-steinke26a} studied tradeoffs between the amount of data needed for statistical estimation and the number of black-box evaluations, while \citet{brown2026privately} obtained improved tradeoffs for monotone statistics. These two works concern scalar outputs, 
and give accuracy guarantees primarily based on distributional assumptions on the input.

\citet{pmlr-v336-rubinstein26a} also address vector-valued estimators through an efficient subsample-test-aggregate reduction from down sensitivity to privacy, with an application to high-dimensional linear regression. Their reduction requires a test that certifies bounded deletion sensitivity and whose acceptance probability cannot decrease too sharply when a record is removed. Our evaluator requires only black-box access to $f$, without a supplied sensitivity bound or stability test. It guarantees privacy for every such $f$ and provides the convex-hull utility guarantee described above.

\paragraph{Paper organization.} Section~\ref{sec:intro-black-box} defines private black-box evaluation in high dimensions and states the main results. Section~\ref{sec:p-qco} develops our private quasi-concave optimization algorithm in high dimensions. Section~\ref{sec:black-box-eval-qco} reduces private black-box evaluation to private quasi-concave optimization. Section~\ref{sec:black-box-qco-equi} establishes reductions in both directions between the two problems, relating their sample complexities. Finally, Section~\ref{sec:applications} presents applications to private center points and user-level quasi-concave and convex optimization, and compares the framework with the sample-and-aggregate baseline.

\section{Private black-box evaluation in high dimensions}\label{sec:intro-black-box}

In this section, we formulate private black-box evaluation in high dimensions and present our main results. Our goal is to privately evaluate an arbitrary function using only black-box access, without assuming a bound on its sensitivity. We express utility through a convex-hull guarantee on the function's outputs over nearby subdatasets. This connects the utility of the private evaluator to the behavior of the underlying function under deletions, allowing us to derive private algorithms from non-private solvers.

The set of natural numbers smaller or equal to $k$ is denoted by $[k]$. We will use $\cU$ to denote the (input) domain set, and $\cU^*$ is the set of all finite multisets that take values from $\cU$. We use two notions of down-neighborhood to describe local access to the data: one bounds the number of deleted elements, while the other specifies the minimum number of retained elements.

\begin{definition}[Down-neighborhood]
    The $k$-down-neighborhood of data set $S\in \cU^*$ is defined by:
    \[
        N^{\downarrow}_{k}(S) = \{S' \subseteq S \text{ such that } |S'|\geq |S| - k\}.
    \]
    The $h$-capped down-neighborhood is defined by:
    \[
        M^{\downarrow}_{h}(S) = \{S' \subseteq S \text{ such that } |S'|\geq h\}.
    \]

\end{definition}

Two data sets $S, S' \in \cU^*$ are neighboring if they differ
by at most one element (which implies that either
 $S' \in N^{\downarrow}_1(S)$ or $S \in N^{\downarrow}_1(S')$).
 
\begin{definition}[Differential privacy]
    Let $\cU,\cY$ be arbitrary domains, $\eps>0$, and $\delta\in (0,1)$ be parameters. A probabilistic algorithm $\cA:\cU^* \to \cY$ 
is said to be $(\eps, \delta)$-differentially private
if for any two neighboring data sets $S, S' \in \cU^*$ and 
any measurable set $A\subseteq \mathbb{R}^d$, we have $\Pr(\cA(S) \in A) \leq e^{\eps} \Pr(\cA(S') \in A) + \delta$
\end{definition}

The sensitivity of a univariate function $f:\cU^* \to \mathbb{R}$ is 
defined by $\Delta_f = \sup |f(S) - f(S')|$ where the supremum
is taken over all neighboring data sets $S, S' \in \cU^*$.

We express the utility of a private black-box evaluator by requiring its output to lie in the convex hull of the function's values on nearby subdatasets. As we will see in our applications, this geometric guarantee allows us to transfer utility guarantees from non-private solvers to their private evaluations. In one dimension, this notion reduces to the interval guarantee of the Sens-o-Matic wrapper of \citet{linder2025privately}.

The convex hull of a set of points $S\subseteq \bR^d$ is the set of all finite convex combinations of points in $S$:
    \[
    \CH(S)=\left\{\sum_{i=1}^{k} w_ix_i \, \bigg| \, k\in \bN, x_i \in S, w_i>0, \sum_{i=1}^{k}w_i =1 \right\}
    \]

\begin{definition}[Convex Down-neighborhood Image (CDI)]
    For any set $S\in \cU^*$, a function $f:\cU^* \to \bR^d$, and $h\in \bN$, we define $\CDI^f_h(S)$ as the convex hull of the image of $M^{\downarrow}_{h}(S)$ under mapping $f$. In other words, $\CDI^f_h(S)= \CH\left(f\left(M^{\downarrow}_{h}(S)\right)\right)$.~\footnote{With slight abuse of notation, we use $f\left(M^{\downarrow}_{h}(S)\right) := \bigcup_{S'\in M^{\downarrow}_{h}(S)} f(S')$ throughout the paper to denote the image of $M^{\downarrow}_{h}(S)$ under mapping $f$}
\end{definition}

The CDI is convex by construction. For fixed $h$, it also expands when elements are added to the dataset, since every previously available subdataset remains available. The following claim formalizes this monotonicity, which is essentially due to the definition (See~\ref{proof:claim:g-monotone} for the proof).

\begin{claim}\label{claim:g-monotone}
    For every $f:\cU^* \to \bR^d$ and $h\in \bN$ the function $\CDI^f_h(.)$ is \textit{monotone} in the sense that
        for every $S_1,S_2 \in \cU^*$ such that $S_1\subseteq S_2$ , we have $\CDI^f_h(S_1)\subseteq \CDI^f_h(S_2)$.
\end{claim}

We now combine privacy, local oracle access, and
convex-hull utility in a single definition. Let $f:\cU^* \to \mathbb{R}^d$ that takes a data set 
and outputs a vector in $\mathbb{R}^d$. A differentially private
evaluator is a probabilistic algorithm $\cA$ that receives 
a data set $S\in \cU^*$, and a function $f$ and
outputs an estimate of $f(S)$. Importantly, $\cA$ has to 
be differentially private with respect to $S$. Moreover, 
in the black-box setting, $\cA$ can only access $f$ through an oracle.

\begin{definition}[Differentially Private Black-box Evaluator]\label{def:black-box}
    Let $\cY\subseteq \bR^d$ be a public finite domain, $\eps>0,\delta\in [0,1),\beta\in(0,1)$, and $\lambda\in \bN$. An $(\eps,\delta)$-differentially private $(\lambda,\beta)$-black-box evaluator is a probabilistic algorithm
     $\cA$ that receives a data set $S\in \cU^*$ and has black-box query access to an
    arbitrary function $f:\cU^* \to \cY$ and outputs a $y\in \mathbb{R}^d$ 
    and has the following properties:
        \begin{itemize}
            \item $\cA$ is a ($\eps, \delta$)-DP with respect to $S$.
            \item $\cA$ can only access $f$ by querying the value of $f(S')$ for any $S' \in N^{\downarrow}_{\lambda}(S)$
            \item $\cA(S)\in \CH(f(N^{\downarrow}_{\lambda}(S)))$ with probability at least $1-\beta$.
        \end{itemize}
\end{definition}

The following theorems give universal privacy wrappers for arbitrary black-box functions. It requires neither a sensitivity bound nor access to the function's implementation, yet guarantees that the private output lies in the convex hull of its values on nearby subdatasets.
\begin{theorem}[Approximate-DP black-box evaluator]\label{thm:main-eval-approx}
     Let $\beta,\eps,\delta\in(0,1)$, and $\cY\subseteq \bR^d$ be a finite domain.
    Then there exists an $(\eps,\delta)$-DP, $(\lambda,\beta)$-black-box evaluator with
    \begin{align*}
    \lambda 
    &= \Tilde{O}\left(d^{2.5} 8^{\log^* |\cY|}\frac{\log(1/\beta\delta)}{\eps}\right).
    \end{align*}
\end{theorem}

\begin{theorem}[Pure-DP black-box evaluator]\label{thm:main-eval-pure}
    Let $\beta,\eps\in(0,1)$, and $\cY\subseteq \bR^d$ be a finite domain.
    Then there exists an $(\eps,0)$-DP, $(\lambda,\beta)$-black-box evaluator with
    \begin{align*}
    \lambda
    &= O\left(\frac{d^{3}\log(|\cY|/\beta)}{\eps}\right).
    \end{align*}
\end{theorem}

\section{Private quasi-concave optimization (QCO)}\label{sec:p-qco}

In this section, we introduce our private quasi-concave optimization algorithm, a key primitive for solving black-box evaluation. We begin by defining quasiconcavity~\footnote{We also describe how to extend a 
quasi-concave function beyond its domain using the canonical extension in Remark~\ref{rem:quasi-func-ext}.}.

\begin{definition}[Quasi-concave function]
    Let $\cY\subseteq \bR^d$. A function $f: \cY \rightarrow \bR$ is quasi-concave if for every $\cA\subseteq\cY$ and every
    $y\in\CH(\cA)\cap\cY$,
    \[
        f(y)\geq\min_{z\in\cA}f(z).
    \]
\end{definition}

When $\cY$ is a convex set, this definition simplifies to requiring that the upper level sets of $f$ be convex, or alternatively that the restriction of $f$ to any line has no local minima. We need the more general definition since our domains are discrete.


In the context of private quasi-concave optimization, let $q:\cU^*\times\bR^d\rightarrow\bR$ be a function with sensitivity $\Delta_q$ in its first argument that is quasi-concave in its second argument. For a public finite candidate domain $\cY\subseteq\bR^d$, the goal is to design a differentially private algorithm that, given a dataset $S\in\cU^*$, returns $\hat{y}\in\CH(\cY)$ whose quality is close to $\max_{y\in\cY}q(S,y)$. We state our results for sensitivity $1$; the general case of $\Delta_q>0$ follows by rescaling.

We develop a 1-dimensional solver with an agnostic utility guarantee: it approximates the actual maximum up to additive error, without requiring its value or a lower bound on it as input. This differs from the multiplicative guarantee relative to a supplied quality promise in $\textsc{RecConcave}$~\cite{beimel2013private}. This is useful when no informative lower bound on the optimal quality is known in advance. The guarantee remains relative to the actual optimum, and, for a fixed sensitivity bound, the additive error is independent of its magnitude.

At a very high-level, our algorithm follows the same recursive idea  of \cite{beimel2013private}: it identifies the scale of an interval on which the quality is close to optimal by expressing scale selection as another QCO problem on a domain of size $O(\log|\cY|)$ and solving it recursively. Once this scale is identified, appropriate private selection procedures locate a point of high quality.

We then extend the solver to high dimensions by optimizing one coordinate at a time. At each iteration, the quality of a candidate coordinate is the maximum quality attainable by completing the remaining coordinates. This approach raises a subtle technical point, which is that we must extend the domain $\cY$ to a superset $\tilde\cY$ as we proceed to ensure that 
each resulting 1-dimensional quality function attains its maximum within its finite candidate domain; see Definition~\ref{def:extension} and the ensuing results. 
The errors accumulate over the $d$ coordinates, contributing a factor of $d$, while advanced composition contributes an additional $\sqrt{d}$ factor, up to logarithmic terms. Similar coordinate-by-coordinate reductions were previously used for private center points and learning halfspaces~\cite{pmlr-v99-beimel19a,kaplan2020private,nissim2026differentially}.

We remark that \citet{cohen2023optimal} also provide agnostic additive-error guarantees for one-dimensional QCO, with sharper domain-size dependence, through a reduction to a cumulatively private interior-point solver. Our coordinate-wise analysis requires explicit dependence on accumulated error, privacy and failure parameters, while their stated QCO guarantee does not fully expose these
parameter dependencies. We therefore give a direct recursive construction following \citet{beimel2013private}, with a self-contained, fully parameterized guarantee in Theorem~\ref{thm:1-dim-opt-new}.

The following theorem formalizes the guarantee of our algorithm. We provide the full construction and proof of the 1-dimensional solver, together with its extension to high dimensions, in Section~\ref{sec:agnostic-qco}.

\begin{theorem}\label{thm:high-dim-opt-1}
Let $\beta,\eps,\delta\in(0,1/2 
)$ be parameters,
$C>0$ be a sufficiently large constant, and let
$\cY\subseteq\bR^d$
be a finite domain with $|\cY|\geq 2$. 
Furthermore, let
$\cQ:\cU^*\times\cY
\rightarrow\bR$
be a discrete quasi-concave function with sensitivity $1$.
Then, there exists an $(\eps,\delta)$-DP algorithm such that, for any
\[
    \Delta\geq C d^{1.5}
    \frac{8^{\log^{*}|\cY|}\log^{*}|\cY|
    \sqrt{\log(1/\delta)}}{\eps}
    \log\left(
    \frac{d\log^{*}|\cY|}{\beta\delta}
    \right),
\]
and every dataset $S\in\cU^*$, its output
$\hat{y}\in\CH(\cY)$ satisfies
\[
    \cQ(S,\hat{y})
    \geq
    \max_{y\in\cY}\cQ(S,y)-\Delta
\]
with probability at least $1-\beta$.
\end{theorem}

\section{Black-box evaluation using QCO}\label{sec:black-box-eval-qco}
In this section, we reduce private black-box evaluation to private quasi-concave optimization. The key is to associate every black-box function $f$ with a quasi-concave quality function of sensitivity at most $1$, without imposing any sensitivity assumption on $f$ itself. We introduce the CDI depth function and show that privately finding a point of sufficiently large depth yields the convex-hull guarantee required for black-box evaluation. This reduction gives our algorithms under approximate and pure differential privacy.

The following function is defined in the spirit of Tukey depth, but on the convex down-neighborhood image of a function.
\begin{definition}[CDI Depth]\label{def:depth}
    Let $\cU$ and $\cY\subseteq \bR^d$ be domain sets. For a set of points $S\in \cU^*$, a point $y \in \cY$, a function $f:\cU^* \to \cY$, and $h\in \bN$, the CDI depth function $\depth^f_h: \cU^* \times \cY \rightarrow \bR$ is defined by:
    \[\depth^f_h(S,y)=\min \left\{|S'|: S' \subseteq S, y\notin \CDI^f_h(S\setminus S')\right\}\]
\end{definition}

\begin{claim}\label{claim:depth-monotone-h}
For every $f:\cU^*\rightarrow\bR^d$, $S\in\cU^*$,
and $y\in\bR^d$, the CDI depth $\depth_h^f(S,y)$
is nonincreasing in $h$. Namely, for every
$h,h'\in\bN$ with $h\leq h'$, we have
$\depth_{h'}^f(S,y)\leq\depth_h^f(S,y)$.
\end{claim}

\begin{claim}\label{claim:depth-quasi}
    For every $f:\cU^* \to \bR^d$, $h\in \bN$ and $S\in \cU^*$ the CDI depth function $\depth^f_h(S,.)$ is quasi-concave w.r.t. its second argument.
    \end{claim}

\begin{claim}\label{claim:depth-sens}
    For every $f:\cU^* \to \bR^d$ and $h\in \bN$ and $y\in \bR^d$, CDI depth function $\depth^f_h(., y)$ has sensitivity $1$ w.r.t. its first argument. More generally, for every $S,S'\subseteq \bR^d$, we have $|\depth^f_h(S,y)-\depth^f_h(S',y)|\leq |S\setminus S'|+|S'\setminus S|$.
\end{claim}

The two preceding claims show that for any choice of $f$ and $h$, the function $\depth^f_h(S, y)$ has sensitivity at most 1 in its first argument and is quasi-concave in its second argument. Therefore, one can privately optimize the value of $y$ using the framework of private quasi-concave optimization. See \ref{proofs:sec:black-box-eval-qco} for the proof of these claims.

The utility guarantee of private quasi-concave optimization is relative to the best score over its finite candidate domain. We must therefore ensure that this domain contains a point with sufficiently large CDI depth. Such a point need not belong to the original output domain $\cY$. We address this by extending $\cY$ to a finite domain $\tilde{\cY}$ satisfying the property defined below.

\begin{definition}[CH-closed domain extension]\label{def:extension}
    Let $\cY \subseteq \bR^d$ be a nonempty set. We say $\Tilde{\cY} \subseteq \CH(\cY)$ is an extension of $\cY$ if for every $\cI \subseteq 2^{\cY}$ the following holds: if  $\cA_{\cI}=\bigcap_{S_i\in\cI}\CH(S_i) \neq \emptyset \quad$ then $\quad\cA_{\cI} \cap \Tilde{\cY}\neq \emptyset$.
\end{definition}

The following lemma shows that the extended domain contains a point with sufficiently large CDI depth. Unlike $\textsc{RecConcave}$~\cite{beimel2013private}, which takes a lower bound on the optimal quality as input, our algorithm does not require such a bound in advance. The bound follows from Helly's theorem and is used only
in the utility analysis of the evaluator and its applications. See~\ref{proof:thm:good-score} for the proof.
\begin{lemma}\label{thm:good-score}
    Let $\cY \subseteq \bR^d$ be a finite domain, and $\Tilde{\cY}$ be any extension of $\cY$. Then for every $f:\cU^* \to \cY$, $h\in \bN$ and $S\in \cU^*$, there exists a point $y^*\in \Tilde{\cY}$ satisfying $\depth^f_h(S,y^*)\geq \frac{|S|-h}{d+1}$.
\end{lemma}

The following two lemmas construct finite extensions of any nonempty finite domain $\cY$ and bound their sizes. In particular, an extension can be chosen with only an additive constant increase in the iterated-logarithmic domain size, which appears in the utility bounds of our quasi-concave optimizers. The first construction follows directly from the definition, while the second requires a more careful geometric analysis. See~\ref{proof:lemma:ext1} and~\ref{proof:lemma:ext2} for the proofs.
\begin{lemma}\label{lemma:ext1}
    For every finite set $\cY \subseteq \bR^d$ there exists an extension $\Tilde{\cY}\subseteq \bR^d$ of size at most~$2^{2^{|\cY|}}$.
\end{lemma}
\begin{lemma}\label{lemma:ext2}
    For every finite set $\cY \subseteq \bR^d$ there exists an extension $\Tilde{\cY}\subseteq \bR^d$ of size at most $|\cY|^{d^2}$.
\end{lemma}

Putting these ingredients together, we obtain our main algorithms for privately optimizing the CDI depth function.

   \begin{algorithm}
\caption{Approximate-DP CDI Depth Optimizer}\label{alg:main}
\begin{algorithmic}[1]
\Require Parameters $\beta,\delta,\eps\in(0,1)$, $h \in \bN$, an input domain $\cU$, a finite output domain $\cY \subseteq \bR^d$, a finite domain extension $\Tilde{\cY}$ of $\cY$, a dataset $S\in \cU^*$, a function $f:\cU^* \to \cY$.

\State Let $\depth_h^f: \cU^* \times \Tilde{\cY} \rightarrow \bR$ be defined as in Definition~\ref{def:depth}.

\State Run the High-dimensional quasi-concave optimization from Theorem~\ref{thm:high-dim-opt-1} with quality function $\depth_h^f$, and parameters $\beta,\eps,\delta$, dataset $S$, input domain $\cU$, output domain $\Tilde{\cY}$, and let the output be $\hat{y}$.

\State \Return $\hat{y}$.
\end{algorithmic}
\end{algorithm}

\begin{theorem}[Main, Approximate-DP]\label{thm:main}
Let $\beta,\eps,\delta\in(0,1/2)$ be parameters,
$C>0$ be a sufficiently large constant,
$\cU$ be an arbitrary domain, and
$\cY\subseteq\bR^d$ be a public finite domain 
with $|\cY|\geq 2$.
Let $\tilde{\cY}$ be a finite extension of $\cY$.
Furthermore, let $f:\cU^*\rightarrow\cY$ be a function,
$h\in\bN$ be a public parameter, and let
$\depth_h^f$ be defined as in Definition~\ref{def:depth}.
Then Algorithm~\ref{alg:main} is $(\eps,\delta)$-DP 
and, for any
\[
    \Delta\geq C d^{1.5}
    \frac{8^{\log^{*}|\tilde{\cY}|}
    \log^{*}|\tilde{\cY}|
    \sqrt{\log(2/\delta)}}{\eps}
    \log\left(
        \frac{d\log^{*}|\tilde{\cY}|}
        {\beta\delta}
    \right),
\]
and every dataset $S\in\cU^*$, its output
$\hat{y}\in\CH(\tilde{\cY})$ satisfies
\begin{align*}
    \depth_h^f(S,\hat{y})
    \geq \max_{y\in\tilde{\cY}}\depth_h^f(S,y)-\Delta
    \geq \frac{|S|-h}{d+1}-\Delta
\end{align*}
with probability at least $1-\beta$.
\end{theorem}

\begin{proof}
    Note that the function $\depth_h^f$ is quasi-concave (Claim~\ref{claim:depth-quasi}) and has sensitivity 1 (Claim~\ref{claim:depth-sens}).
    Also, using Theorem~\ref{thm:good-score}, we know that for every $S\in\cU^*$, and any extension $\Tilde{\cY}$ of $\cY$, it holds that $\max_{y\in\Tilde{\cY}} \depth_h^f\left(S,y\right) \geq \frac{|S|-h}{d+1}$. Putting together, and applying Theorem~\ref{thm:high-dim-opt-1} concludes the result.
\end{proof}

\begin{algorithm}
\caption{Approximate-DP Black-Box Evaluator}
\label{alg:noisy-cutoff-eval}
\begin{algorithmic}[1]
\Require Parameters $\beta,\eps,\delta\in(0,1)$,
a finite domain $\cY\subseteq\bR^d$ with
$|\cY|\geq2$, an input dataset $S\in\cU^*$,
and oracle access to $f:\cU^*\to\cY$.

\State Choose an extension $\tilde{\cY}$ using
Lemma~\ref{lemma:ext1}.

\State Set $\beta_0=\min\{\beta,1/4\}$,
$\eps_0=\eps/2$, and $\delta_0=\delta/2$.

\State Choose $\Delta$ satisfying the error bound in
Theorem~\ref{thm:main} for parameters
$\beta_0,\eps_0,\delta_0$ and domain $\tilde{\cY}$.

\State Set $R=\frac{1}{\eps_0}
      \log\left(1+\frac{e^{\eps_0}-1}{2\delta_0}\right)$,
    $K=\left\lceil(d+1)(\Delta+1)\right\rceil$.

\State Sample
$Z\sim\operatorname{TLap}(1,\eps_0,\delta_0)$
and compute $\tilde n=|S|+Z$. \Comment{Truncated Laplace (Definition~\ref{def:truncated-laplace}).}

\State Set $h=\lfloor\tilde n-R-K\rfloor$.

\If{$h\leq0$}
    \State \Return $f(\emptyset)$.
\EndIf

\State Run CDI Depth Optimizer (Algorithm~\ref{alg:main}) with cutoff $h$,
parameters $\beta_0,\eps_0,\delta_0$,
dataset $S$, function $f$, and extension $\tilde{\cY}$.

\State \Return its output.
\end{algorithmic}
\end{algorithm}

We are now ready to introduce our approximate-DP black-box evaluator in Algorithm~\ref{alg:noisy-cutoff-eval}, which satisfies the privacy and utility guarantees of Theorem~\ref{thm:main-eval-approx}. A private estimate of the dataset size determines whether to optimize the CDI depth or return $f(\emptyset)$. The cutoff ensures that all black-box calls use permitted subdatasets and that depth optimization yields the required utility with high probability. See Appendix~\ref{proof:thm:main-eval-approx} for the formal proof.


    \begin{algorithm}
\caption{Pure-DP CDI Depth Optimizer}\label{alg:main-pure}
\begin{algorithmic}[1]
\Require Parameters $\beta\in(0,1)$, $\eps>0$, $h \in \bN$, an input domain $\cU$, a finite output domain $\cY \subseteq \bR^d$, a finite domain extension $\Tilde{\cY}$ of $\cY$, a dataset $S\in \cU^*$, a function $f:\cU^* \to \cY$.

\State Let $\depth_h^f: \cU^* \times \Tilde{\cY} \rightarrow \bR$ be defined as in Definition~\ref{def:depth}.

\State Run the Exponential Mechanism~\ref{thm:exp-mech} with quality function $\depth_h^f$, and parameters $\beta,\eps$, dataset $S$, input domain $\cU$, output domain $\Tilde{\cY}$, and let the output be $\hat{y}$.

\State \Return $\hat{y}$.
\end{algorithmic}
\end{algorithm}

\begin{theorem}[Main, Pure-DP]\label{thm:main-pure}
Let $\beta\in(0,1),\eps>0$ be parameters,
$\cU$ be an arbitrary domain, and
$\cY\subseteq\bR^d$ be a public finite domain.
Let $\tilde{\cY}$ be a finite extension of $\cY$.
Furthermore, let $f:\cU^*\rightarrow\cY$ be a function,
$h\in\bN$ be a public parameter, and let
$\depth_h^f:\cU^*\times\tilde{\cY}\rightarrow\bR$
be defined as in Definition~\ref{def:depth}.
Then Algorithm~\ref{alg:main-pure} is $(\eps,0)$-DP.
Furthermore, for any
\[
    \Delta\geq
    \frac{2\log\left(|\tilde{\cY}|/\beta\right)}{\eps},
\]
and every dataset $S\in\cU^*$, its output
$\hat{y}\in\tilde{\cY}$ satisfies
\begin{align*}
    \depth_h^f(S,\hat{y})
    \geq \max_{y\in\tilde{\cY}}\depth_h^f(S,y)-\Delta
    \geq \frac{|S|-h}{d+1}-\Delta
\end{align*}
with probability at least $1-\beta$.
\end{theorem}
\begin{proof}
    Note that the function $\depth_h^f$ has sensitivity 1 (Claim~\ref{claim:depth-sens}).
Also, using Theorem~\ref{thm:good-score}, we know that for every $S\in\cU^*$, and any extension $\Tilde{\cY}$ of $\cY$, it holds that $\max_{y\in\Tilde{\cY}} \depth_h^f\left(S,y\right) \geq \frac{|S|-h}{d+1}$. Putting together and applying the Exponential Mechanism (Theorem~\ref{thm:exp-mech}) concludes the result.
\end{proof}

\begin{algorithm}
\caption{Pure-DP Black-Box Evaluator}
\label{alg:noisy-cutoff-eval-pure}
\begin{algorithmic}[1]
\Require Parameters $\beta\in(0,1)$, $\eps>0$,
a finite domain $\cY\subseteq\bR^d$ with
$|\cY|\geq2$, an input dataset $S\in\cU^*$,
and oracle access to $f:\cU^*\to\cY$.

\State Choose an extension $\tilde{\cY}$ using
Lemma~\ref{lemma:ext2}, and let $M=|\tilde{\cY}|$.

\State Set $\eta=\beta/2$, $\beta_0=\beta/2$,
$\eps_0=\eps/2$, and
$\delta_0=\min\left\{
    1/4,\,
    \eta(e^{\eps_0}-1)/M
\right\}$.

\State Set
$\Delta=\frac{2\log(M/\beta_0)}{\eps_0}$,
$K=\left\lceil(d+1)(\Delta+1)\right\rceil$, and
$R=\frac{1}{\eps_0}
\log\left(1+\frac{e^{\eps_0}-1}{2\delta_0}\right)$.

\State Sample $B\sim\operatorname{Bernoulli}(\eta)$.

\If{$B=1$}
    \State \Return a uniformly random point of
    $\tilde{\cY}$.
\EndIf

\State Sample
$Z\sim\operatorname{TLap}(1,\eps_0,\delta_0)$
and compute $\tilde n=|S|+Z$.

\State Set $h=\lfloor\tilde n-R-K\rfloor$.

\If{$h\leq0$}
    \State \Return $f(\emptyset)$.
\EndIf

\State Run the Pure-DP CDI Depth Optimizer
(Algorithm~\ref{alg:main-pure}) with cutoff $h$,
parameters $\beta_0,\eps_0$,
dataset $S$, function $f$, and extension
$\tilde{\cY}$.

\State \Return its output.
\end{algorithmic}
\end{algorithm}

Algorithm~\ref{alg:noisy-cutoff-eval-pure} presents our pure-DP black-box evaluator, which satisfies the privacy and utility guarantees of Theorem~\ref{thm:main-eval-pure}. As before, a private estimate of the dataset size determines whether to optimize the CDI depth or return $f(\emptyset)$. With a small, carefully chosen probability, the algorithm instead returns a uniformly random point from the extended domain. This additional randomization ensures pure differential privacy while preserving the utility guarantee. See Appendix~\ref{proof:thm:main-eval-pure} for the formal proof.

\section{Equivalency of Private Black-box Evaluation and QCO}
\label{sec:black-box-qco-equi}
In this section, we show that, up to an $O(d)$ factor in the error and deletion parameters and an additive constant in the iterated-logarithmic domain size, private QCO is complete for private black-box evaluation. A private QCO solver suffices to privately evaluate every function $f:\cU^*\rightarrow\cY$ under our convex-hull guarantee. Conversely, an evaluator that works for every such function must also solve private QCO. In this sense, the hardest instances of private black-box evaluation already arise from optimizers of quasi-concave quality functions.  We first introduce notation for the sample complexity of private QCO and private black-box evaluation.

\begin{definition}[Sample complexity of private QCO]
\label{def:qco-samples}
Let $\alpha,\beta,\delta,\eps\in(0,1)$, and let $d,L\geq1$ be integers.
Define
$n_{\mathrm{QCO}}(\alpha,\beta,\eps,\delta,d,L)$
as the smallest integer $N\geq1$ such that~\footnote{
If no such $N$ exists, set
$n_{\mathrm{QCO}}(\alpha,\beta,\eps,\delta,d,L)
=\infty$.},
for every input domain $\cU$ and every nonempty finite
domain $\cY\subseteq\bR^d$ satisfying
$\log^*|\cY|\leq L$, there exists a randomized oracle
algorithm $\cA$ with the following properties.

For every discrete quasi-concave quality function
$\cQ:\cU^*\times\cY\to\bR$ with sensitivity at most $1$:
\begin{enumerate}
    \item The algorithm $\cA(\cQ,\cdot)$ is
    $(\eps,\delta)$-DP on $\cU^*$.

    \item On every dataset $S\in\cU^*$, the algorithm
    outputs a point in $\CH(\cY)$ and accesses $\cQ$
    only through queries $\cQ(T,y)$ with
    $T\subseteq S$ and $y\in\cY$, on every run.

    \item For every dataset $S\in\cU^*$ satisfying
    $|S|\geq N$, its output $\hat y$ satisfies
    \[
        \cQ(S,\hat y)
        \geq
        \max_{y\in\cY}\cQ(S,y)-\alpha N
    \]
    with probability at least $1-\beta$.
\end{enumerate}
\end{definition}

\begin{definition}
[Sample complexity of private black-box evaluation]
\label{def:black-box-samples}
Let $\alpha,\beta,\delta,\eps\in(0,1)$, and let $d,L\geq1$ be integers.
Define
$n_{\mathrm{black\text{-}box}}
(\alpha,\beta,\eps,\delta,d,L)$
as the smallest integer $N\geq1$ such that~\footnote{
If no such $N$ exists, set
$n_{\mathrm{black\text{-}box}}
(\alpha,\beta,\eps,\delta,d,L)=\infty$.},
for every input domain $\cU$ and every nonempty finite
domain $\cY\subseteq\bR^d$ satisfying
$\log^*|\cY|\leq L$, there exists a randomized oracle
algorithm $\cA$ with the following properties.

For every fixed function $f:\cU^*\to\cY$:
\begin{enumerate}
    \item The algorithm $\cA(f,\cdot)$ is
    $(\eps,\delta)$-DP on $\cU^*$.

    \item On every dataset $S\in\cU^*$, the algorithm
    accesses $f$ only through queries $f(T)$ with
    $T\in N^\downarrow_{\lfloor\alpha N\rfloor}(S)$, on every run.

    \item For every dataset $S\in\cU^*$ satisfying
    $|S|\geq N$, its output $\hat y$ satisfies
    \[
        \hat y\in
        \CH\left(f(N^\downarrow_{\lfloor\alpha N\rfloor}(S))\right)
    \]
    with probability at least $1-\beta$.
\end{enumerate}
\end{definition}

The following theorem formalizes the equivalence discussed above.
See~\ref{proof:thm:black-box-qco-equi} for the proof.

\begin{theorem}
[From QCO to black-box evaluation and back]
\label{thm:black-box-qco-equi}
There exists a universal constant $C>0$ such that
the following holds.
Let $\beta,\eps\in(0,1)$,
$\delta,\alpha\in(0,1/2)$, and let $d,L\geq1$ be integers.
Then,
\begin{align*}
n_{\mathrm{QCO}}
    (2\alpha,\beta,\eps,\delta,d,L)
    &\leq
    n_{\mathrm{black\text{-}box}}
    (\alpha,\beta,\eps,\delta,d,L)\\
&\leq
\max\left\{
n_{\mathrm{QCO}}
    \left(
        \frac{\alpha}{2(d+1)},
        \beta,\frac{\eps}{2},\frac{\delta}{2},
        d,L+2
    \right),
\left\lceil
    \frac{C}{\alpha}
    \left(
        d+\frac{\log(1/\delta)}{\eps}
    \right)
\right\rceil
\right\}.
\end{align*}
\end{theorem}

\section{Instantiations}\label{sec:applications}
In this section, we instantiate our main black-box evaluator with non-private solvers to obtain private algorithms for center points and user-level quasi-concave and convex optimization.
These applications illustrate how our guarantees translate into utility bounds using only black-box access to the underlying solvers.

\subsection{Recovering 1-dimensional black-box of \cite{linder2025privately}}

For $d=1$, our convex-hull guarantee recovers the one-dimensional interval guarantee of the Sens-o-Matic wrapper of \cite{linder2025privately}, with our stated locality and failure parameters. Consequently, the same application reductions apply to our evaluator, including private mean estimation with error that adapts to the scale of the data and user-level convex optimization in one dimension. In these applications, utility follows by bounding how much the non-private statistic changes when a small number of records or users are removed.

\begin{theorem}[Sens-o-Matic Privacy Wrapper, \citet{linder2025privately}]
    Let $\eps,\delta,\beta\in (0,1)$, $\lambda\in \bN$, and $\cY\subseteq \bR$ be a finite domain. Then:
    \begin{itemize}
        \item There exists a \textbf{1-dimensional} $(\eps,\delta)$-DP $(\lambda,\beta)$-black-box evaluator with
    $\lambda = 2^{O\left(\log^* |\cY|\right)}\frac{\log(1/\beta\delta)}{\eps}$.

    \item There exists a \textbf{1-dimensional} $(\eps,0)$-DP $(\lambda,\beta)$-black-box evaluator with
    $\lambda 
    = O\left(\frac{\log(|\cY|/\beta)}{\eps}\right)$.
    \end{itemize}
\end{theorem}

\subsection{Private center point}
A natural geometric application of our black-box evaluator is private center point computation.

Center points generalize the median to higher dimensions: every closed halfspace containing a center point of $S$ contains at least $|S|/(d+1)$ points (See Theorem~\ref{thm:center-point-org}). Finding such points privately is a basic problem in private geometry, with applications to learning halfspaces. The next, is the formal definition of Tukey depth.

\begin{definition}[Tukey depth]\label{def:tukey-depth}
For a dataset $S\in(\bR^d)^*$ and a point $y\in\bR^d$,
the Tukey depth of $y$ with respect to $S$ is defined by
\[
    \TD(S,y)
    =
    \min\left\{
    |S\cap H|:
    H\subseteq\bR^d \text{ is a closed halfspace containing } y
    \right\},
\]
where points are counted with multiplicity.
\end{definition}

\citet{pmlr-v99-beimel19a} give an $(\eps,\delta)$-DP algorithm achieving Tukey depth at least $|S|/(2(d+1))$ with sample complexity $\tilde{O}\left(d^{2.5}2^{O(\log^*|\cY|)}/\eps\right)$. Our framework recovers this sample-complexity bound and yields an explicit additive approximation to the maximum Tukey depth.  

A useful observation is that Tukey depth coincides with a special case of CDI depth for \emph{internal} functions, whose outputs lie in the convex hull of their inputs. The following lemma formalizes this connection and is particularly useful for analyzing black-box evaluation of such functions.

A function $f: (\bR^d)^* \rightarrow \bR^d$ is called \emph{internal} if $f(S)\in \CH(S)$ for any $S\in (\bR^d)^*$. These functions are also often referred to as location statistics (e.g. mean, median).

\begin{lemma}\label{lemma:depth-td}
    Let $\cU\subseteq \bR^d$ be a domain, and $S\in \cU^*$ be a dataset. For any internal function $f:\cU\rightarrow \bR^d$, any $h\in \bN$, and $y\in \bR^d$, we have $\depth_h^f(S,y)\leq \TD(S,y)=\depth_1^f(S,y)$.
\end{lemma}
See \ref{proof:lemma:depth-td} for the proof.

Let $\cY\subseteq\bR^d$ be a nonempty finite domain.
Define the black-box function
$f_{\operatorname{center}}:\cY^*\rightarrow\cY$ by
\[
    f_{\operatorname{center}}(S)
    \in
    \argmax_{y\in\cY}\TD(S,y)
\]
for every $S\in\cY^*$, where ties are broken
deterministically. Observe that
$f_{\operatorname{center}}$ is internal on nonempty
datasets. This allows us to use the Tukey depth
characterization of $\depth_h^{f_{\operatorname{center}}}$
from the previous lemma, yielding the following result.
\begin{corollary}[Center point]\label{lemma:tm-improved}
Let $\cY\subseteq\bR^d$ be a finite domain, and $\tilde{\cY}$ be its extension from Lemma~\ref{lemma:ext1}.
$\beta,\eps,\delta\in(0,1)$, and let
$f_{\operatorname{center}}$ be defined as above.
Then the algorithm $\cA$ from~\ref{alg:main}
is $(\eps,\delta)$-DP, and for every dataset
$S\in\cY^*$, its output
$\hat{y}:=\cA(f_{\operatorname{center}},S,\beta,\eps,\delta,\cY,\tilde{\cY})
\in\bR^d$
satisfies
\[
    \TD(S,\hat{y})
    \geq
    \max_{y\in\bR^d}\TD(S,y)-\Delta
    \geq
    \frac{|S|}{d+1}-\Delta
\]
with probability at least $1-\beta$, where $\Delta=\tilde{O}\left(
    d^{1.5}
    \frac{8^{\log^*|\cY|}}{\eps}
    \log\left(\frac{1}{\beta\delta}\right)
    \right)$.
\end{corollary}

See \ref{proof:lemma:tm-improved} for the proof.

\subsection{From convex-hull guarantees to optimization}  \label{sec:black-box-opt}

A useful observation is that the convex-hull guarantee of a black-box evaluator translates directly into optimization guarantees. For quasi-concave quality functions and convex losses, the private output is at least as good as the worst black-box output over nearby subdatasets, with all outputs evaluated on the original dataset $S$. The following claim formalizes this observation.

\begin{claim}\label{claim:generic-opt-black-box} 
Let $\cU$ be an input domain, $\cY\subseteq\bR^d$ be a finite output domain, and $\cA$ be an $(\eps,\delta)$-DP, $(\lambda,\beta)$-black-box evaluator. Then, for \textbf{every} black-box function $f:\cU^*\rightarrow\cY$ and dataset $S\in\cU^*$, with probability at least $1-\beta$, its output $\hat{y}=\cA(f,S)$ satisfies the following simultaneously: 

\begin{quote} \textit{(1)} $\cQ(S,\hat{y})\geq \min_{S'\in N^{\downarrow}_{\lambda}(S)} \cQ(S,f(S'))$ for \textbf{any} quality function $\cQ:\cU^*\times\bR^d\rightarrow\bR$ that is quasi-concave in its second argument.

\textit{(2)} $\cL(S,\hat{y})\leq \max_{S'\in N^{\downarrow}_{\lambda}(S)} \cL(S,f(S'))$ for \textbf{any} loss function $\cL:\cU^*\times\bR^d\rightarrow\bR$ that is convex in its second argument. 
\end{quote} 
\end{claim}  

These guarantees require no sensitivity bound on $\cQ$ or $\cL$. Moreover, the evaluator need not know or evaluate either function; it only requires black-box access to $f$. To apply the claim, it therefore suffices to control the quality of $f(S')$, measured on $S$, when at most $\lambda$ elements are deleted. See~\ref{proof:claim:generic-opt-black-box} for the proof.  

In the next two subsections, we directly instantiate this claim for user-level quasi-concave and convex optimization, treating each user's dataset as a single element of the input domain.

\subsection{User-Level Private Quasi-Concave Optimization}  

A useful application of our black-box evaluator is quasi-concave optimization when no finite bound on the sensitivity of the quality function is available. Standard sensitivity-based private QCO algorithms require such a bound to rescale the objective to have sensitivity at most $1$. Without this bound, the same approach cannot be applied directly. Our black-box evaluator avoids this requirement: it only accesses a non-private optimizer, and requires no bound on the range or sensitivity of the quality function.  

This setting arises naturally in user-level privacy, where each user may hold an arbitrary number of data points. Even when each individual data point has a bounded effect on the objective, adding or removing one user's entire dataset may change its value by an arbitrarily large amount. By treating each user's dataset as a single element of the input domain, our evaluator provides user-level privacy without imposing a uniform bound on the users' contribution sizes. Its utility guarantee compares the quality of the private output with that of non-private solutions obtained after removing at most $\lambda$ users, with all solutions evaluated on the original dataset.  

We start by recalling user-level differential privacy.

\begin{definition}[User-Level Privacy]\label{def:user-level}
Let $\eps,\delta\in(0,1)$ be parameters. An algorithm $\cA$ satisfies $(\eps,\delta)$-user-level differential privacy if it satisfies $(\eps,\delta)$-differential privacy with respect to the neighboring relation where two datasets $S,S'\in((\bR^d)^*)^*$ are considered neighbors if one can be obtained from the other by adding or removing a single user's entire dataset.
\end{definition}

We now formally define the user-level quasi-concave
optimization setting and state the resulting guarantee.

Let $\cX$ be an input domain and
$\cY\subseteq\bR^d$ be a public finite nonempty output domain.
Let $\cQ:(\cX^*)^*\times\bR^d\rightarrow\bR$
be an arbitrary quality function such that
$\cQ(T,\cdot)$ is quasi-concave for every users' dataset
$T\in(\cX^*)^*$.
Define the black-box optimizer
$f_{\cQ}:(\cX^*)^*\rightarrow\cY$ by
$f_{\cQ}(T)\in\argmax_{y\in\cY}\cQ(T,y)$,
where ties are broken deterministically.

\begin{corollary}[User-level quasi-concave optimization]
\label{cor:user-level-quasi-opt}
Let $\beta,\eps,\delta\in(0,1)$ be parameters,
and let $\cQ$ and $f_{\cQ}$ be defined as above.
Then, the black-box evaluator from~\ref{thm:main-eval-approx}
is $(\eps,\delta)$-user-level DP.
Furthermore, for every users' dataset
$S=(S_1,S_2,\cdots,S_n)\in(\cX^*)^*$,
if
$n\geq\lambda=\tilde{O}\left(
    d^{2.5}
    \frac{8^{\log^*|\cY|}}{\eps}
    \log\left(\frac{1}{\beta\delta}\right)
\right)$,
its output $\hat{y}=\cA(f_{\cQ},S)$ satisfies
\begin{align*}
    \cQ(S,\hat{y})
    \geq
    \min_{S'\in N^{\downarrow}_{\lambda}(S)}
    \cQ(S,f_{\cQ}(S'))
\end{align*}
with probability at least $1-\beta$.
\end{corollary}

This is a direct consequence of Claim~\ref{claim:generic-opt-black-box} and Theorem~\ref{thm:main-eval-approx}, treating each user's dataset $S_i\in\cX^*$ as a single element of the input domain. The minimum in the guarantee accounts for the worst choice of at most $\lambda$ users to remove, with every resulting solution evaluated on the original dataset $S$. Thus, whenever these non-private solutions retain high quality on $S$, the private output does as well. This holds without bounding the range or sensitivity of $\cQ$, even when users hold different amounts of data and contribute very differently to the objective.  

A useful baseline is the sample-and-aggregate framework, which runs the non-private optimizer on disjoint blocks of users and privately aggregates the resulting solutions. We compare this baseline with our black-box evaluator in subsection~\ref{subsec:sample-agg}.

We next instantiate this guarantee for the user-level center point problem.

Let $\cY\subseteq\bR^d$ be a finite domain,
and let $\tilde{\cY}$ be a finite extension of $\cY$.
For a users' dataset
$S=(S_1,S_2,\cdots,S_n)\in(\cY^*)^*$,
define the flattened dataset by
$\operatorname{flat}(S):=\biguplus_{i\in[n]}S_i$,
where the union counts multiplicities.
Define the center point solver
$f_{\operatorname{center}}:(\cY^*)^*\rightarrow\tilde{\cY}$
by
\[
    f_{\operatorname{center}}(S)
    \in\argmax_{y\in\tilde{\cY}}
    \TD\left(\operatorname{flat}(S),y\right),
\]
where ties are broken deterministically.

\begin{corollary}[User-level center point]
\label{cor:user-level-center}
Let $\beta,\eps,\delta\in(0,1)$ be parameters,
$\cY\subseteq\bR^d$ be a finite domain,
and $\tilde{\cY}$ be a finite extension of $\cY$
from Lemma~\ref{lemma:ext1}.
Let
$f_{\operatorname{center}}$
be the black-box optimizer defined above.
Then the black-box evaluator $\cA$  from~\ref{thm:main-eval-approx}
is $(\eps,\delta)$-user-level DP.

Furthermore, let $S=(S_1,S_2,\cdots,S_n)$ be any users'
dataset, where $S_i\in\cY^{m_i}$ for every $i\in[n]$.
Let $M=\sum_{i\in[n]}m_i$ be the total number of points,
and let $m_{i_1}\geq m_{i_2}\geq\cdots\geq m_{i_n}$
denote the users' contribution sizes in non-increasing order.
Then, if
$n\geq\lambda=\tilde{O}\left(
        d^{2.5}
        \frac{8^{\log^*|\cY|}}{\eps}
        \log\left(\frac{1}{\beta\delta}\right)
    \right)$,
its output $\hat{y}=\cA(f_{\operatorname{center}},S)$
satisfies
\begin{align*}
    \TD\left(\operatorname{flat}(S),\hat{y}\right)
    \geq
    \max\Bigg\{
        \max_{y\in\bR^d}
        \TD\left(\operatorname{flat}(S),y\right)
        -\sum_{j\in[\lambda]}m_{i_j},
        \frac{M-\sum_{j\in[\lambda]}m_{i_j}}{d+1}
    \Bigg\}
\end{align*}
with probability at least $1-\beta$.
In particular, if $m_i=m>0$ for every $i\in[n]$,
then, on the same event,
\[
    \TD\left(\operatorname{flat}(S),\hat{y}\right)
    \geq \frac{(n-\lambda)m}{d+1}.
\]
\end{corollary}
See~\ref{proof:cor:user-level-center} for the proof.

The two terms in the bound provide complementary
guarantees. The first compares the output with the maximum Tukey depth over $\cY$, with an additive loss
of at most $\sum_{j\in[\lambda]}m_{i_j}$, the total number of points held by the $\lambda$ largest contributors. The second preserves the usual
$1/(d+1)$ center point guarantee for the number of points remaining after removing these users.

For comparison, the ordinary center point guarantee from Corollary~\ref{lemma:tm-improved}, applied to the flattened dataset of size $M$, gives depth at least $M/(d+1)-\Delta$ under item-level privacy. Here, privacy
protects each user's entire dataset, and the guaranteed
depth is at least
$\frac{M}{d+1}
    -
    \frac{\sum_{j\in[\lambda]}m_{i_j}}{d+1}$.
    
Thus, the loss relative to the usual center point benchmark is controlled explicitly by the largest users' contributions, without requiring a uniform upper bound on their sizes.

In particular, when $m_i=m>0$ for every $i\in[n]$, we have $\TD(\operatorname{flat}(S),\hat{y})     \geq     \frac{(n-\lambda)m}{d+1}$.  

For fixed dimension, output domain, and privacy and failure parameters, $\lambda$ is independent of $m$. Thus, with $n>\lambda$ fixed, increasing the number of points contributed by each user increases the guaranteed depth linearly, without requiring more users. The privacy guarantee continues to protect each user's entire dataset, regardless of its size.  This illustrates the user-level nature of the result: the required number of users is independent of $m$, while the guaranteed depth grows with their total contribution. 
See \ref{proof:cor:user-level-center} for the formal proof.

\subsection{User-Level Private Convex Optimization}  

Our black-box evaluator also provides a generic approach to user-level private convex optimization. It applies to arbitrary convex user loss functions and convex regularizers, without requiring a uniform bound on the users' contribution sizes or on the sensitivity of the objective. This includes objectives formed from bounded per-example loss terms when each user may contribute arbitrarily many examples.  

The evaluator $\cA$ only requires black-box access to a non-private optimizer with outputs in the finite domain $\cY$. It need not know or evaluate the loss $\cL$, and requires no access to gradients or subgradients. Moreover, no known Lipschitz bound, smoothness, or differentiability assumption is needed.  

For comparison, representative prior results on user-level private convex optimization \cite{pmlr-v202-ghazi23a,pmlr-v238-liu24g} analyze Lipschitz losses with a fixed number of examples per user; their stochastic optimization guarantees also assume i.i.d. sampling. These works already accommodate nonsmooth losses. Our corollary provides a dataset-wise empirical-loss guarantee for arbitrary user datasets, with the excess loss controlled by the loss variation associated with at most $\lambda$ deleted users. Bounds on individual loss terms can then be used to quantify this error, without being provided to the evaluator.

Another useful baseline is the sample-and-aggregate framework, which runs the non-private optimizer on disjoint blocks of users and privately aggregates the resulting solutions. We compare this baseline with our black-box evaluator in subsection~\ref{subsec:sample-agg}.

We now formally define the user-level convex optimization
setting and state the resulting guarantee.

Let $\cX$ be an input domain and
$\cY\subseteq\bR^d$ be a finite output domain.
Let $\phi_i:\cX^*\times\bR^d\rightarrow\bR$ be individual
loss functions that are convex in their second argument,
and let $r:\bR^d\rightarrow\bR$ be a convex regularizer.
For a users' dataset
$S=(S_1,S_2,\ldots,S_n)\in(\cX^*)^*$,
define the empirical risk function
$\cL:(\cX^*)^*\times\bR^d\rightarrow\bR$
and the black-box optimizer
$f_{\cL}:(\cX^*)^*\rightarrow\cY$ by
\begin{align*}
    \cL(S,y)
    &=\sum_{i\in[n]}\phi_i(S_i,y)+r(y),
    &
    f_{\cL}(S)
    &\in\argmin_{y\in\cY}\cL(S,y),
\end{align*}
where ties are broken deterministically.
For a subdataset $S'\subseteq S$, define
$\cL(S',y)=\sum_{S_i\in S'}\phi_i(S_i,y)+r(y)$,
using the original user indices.

\begin{corollary}\label{cor:user-level-co}
Let $\beta,\eps,\delta\in(0,1)$ be parameters,
and let $\cL$ and $f_{\cL}$ be defined as above.
Then, the black-box evaluator from
Theorem~\ref{thm:main-eval-approx}
is $(\eps,\delta)$-user-level DP.

Furthermore, let $S=(S_1,S_2,\ldots,S_n)$ be any users'
dataset, where $S_i\in\cX^{m_i}$ for every $i\in[n]$.
Then, if $n\geq\lambda=\tilde{O}\left(
        d^{2.5}
        \frac{8^{\log^*|\cY|}}{\eps}
        \log\left(\frac{1}{\beta\delta}\right)
    \right)$,
the output $\hat{y}=\cA(f_{\cL},S)$ satisfies
\begin{align*}
    \cL(S,\hat{y})
    &\leq
    \max_{S'\in N^{\downarrow}_{\lambda}(S)}
    \cL(S,f_{\cL}(S'))\\
    &\leq
    \min_{y\in\cY}\cL(S,y)
    +
    \max_{S'\in N^{\downarrow}_{\lambda}(S)}
    \sum_{S_i\in S\setminus S'}
    \left[
        \phi_i(S_i,f_{\cL}(S'))
        -
        \phi_i(S_i,f_{\cL}(S))
    \right]
\end{align*}
with probability at least $1-\beta$.
\end{corollary}
This is essentially due to Claim~\ref{claim:generic-opt-black-box} and Theorem~\ref{thm:main-eval-approx}. See \ref{proof:cor:user-level-co} for the formal proof.

\subsection{Comparison with the sample-and-aggregate baseline}\label{subsec:sample-agg}  

A natural baseline for private black-box evaluation is the sample-and-aggregate framework of \citet{nissim2007smooth}. This framework applies a non-private procedure to smaller blocks of data and privately aggregates the resulting outputs. For quasi-concave quality functions and convex loss
functions, a private center point algorithm is a
natural choice of aggregator. The convex-hull guarantee of the aggregator preserves any lower bound on quality, or upper bound on loss, satisfied by all block outputs.  

Let $S\in\cU^*$ be a dataset of size $n$, and let
$f:\cU^*\rightarrow\cY$ be an arbitrary black-box function.

The sample-and-aggregate baseline randomly partitions $S$ into
$\lambda=\Tilde{O}\left(d^{2.5} 8^{\log^* |\cY|}\frac{\log(1/\beta\delta)}{\eps}\right)$\footnote{With a sufficiently large
constant for both our evaluator (Theorem~\ref{thm:main-eval-approx}) and private center point
aggregation of $\lambda$ points (Corollary~\ref{lemma:tm-improved}).} disjoint blocks $T_1,\ldots,T_\lambda$, using a
data-independent assignment under which a neighboring
change affects at most one block. It evaluates $f$ on
each block and privately aggregates the outputs using
a center point (Corollary~\ref{lemma:tm-improved}). With probability at least $1-\beta$,
the resulting output satisfies
\[
    \hat{y}_{\mathrm{SA}}
    \in\CH\left(\{f(T_1),\ldots,f(T_\lambda)\}\right).
\]
Consequently, for any quasi-concave quality function
$\cQ$ and any convex loss function $\cL$, we have
\begin{align*}
    \cQ(S,\hat{y}_{\mathrm{SA}})
    &\geq
    \min_{i\in[\lambda]}\cQ(S,f(T_i)),\\
    \cL(S,\hat{y}_{\mathrm{SA}})
    &\leq
    \max_{i\in[\lambda]}\cL(S,f(T_i)).
\end{align*}
On the other hand, let $\hat{y}_{\mathrm{BB}}=\cA(f,S)$
denote the output of our $(\lambda,\beta)$-black-box
evaluator (Theorem~\ref{thm:main-eval-approx}). By Claim~\ref{claim:generic-opt-black-box},
with probability at least $1-\beta$, we have
\begin{align*}
    \cQ(S,\hat{y}_{\mathrm{BB}})
    &\geq
    \min_{S'\in N^{\downarrow}_{\lambda}(S)}
    \cQ(S,f(S')),\\
    \cL(S,\hat{y}_{\mathrm{BB}})
    &\leq
    \max_{S'\in N^{\downarrow}_{\lambda}(S)}
    \cL(S,f(S')).
\end{align*}

For balanced blocks, the baseline therefore relies on
solutions computed on approximately $n/\lambda$ elements,
whereas our guarantee is in terms of solutions computed
on at least $n-\lambda$ elements. In particular, suppose
that $f(T)$ achieves a desired quality or loss, measured
on $S$, for every $T\subseteq S$ with $|T|\geq h$.
Then sufficient dataset sizes for the two reductions are
\[
    n_{\mathrm{SA}}=\lambda h,
    \qquad
    n_{\mathrm{BB}}=h+\lambda.
\]
Thus, under this uniform-subset guarantee, the baseline
incurs a \textbf{multiplicative} overhead, while our
evaluator incurs an \textbf{additive} overhead. This comparison applies to both item-level and user-level privacy,
by taking $\cU=\cX$ or $\cU=\cX^*$, respectively.

\section*{AI Disclosure}
The authors proved all theorems and wrote all proofs in this work. AI (GPT-Sol 5.6) was used to check the proofs and identified some edge cases in the algorithms and theorems, which the authors subsequently addressed. AI was also used to refine the language and grammar.

\printbibliography

\appendix

\section{Preliminaries}
This section collects the differential privacy tools and geometric facts used in our constructions and analyses.
\subsection{Privacy tool-box}
\begin{theorem}[Advanced composition, \citet{dwork2010boosting}]\label{thm:advanced-composition}
    Let $\eps,\delta>0$, and $M_1,\dots,M_t : \cU^* \rightarrow \cY$ be $(\eps,\delta)$-differentially private mechanisms. Then the mechanism that takes $S \in \cX^n$ as input and outputs $(M_1(S),\dots,M_t(S))$ is $(\eps',t\delta + \delta')$-differentially private, where $\eps' = \eps\sqrt{2t\log(1/\delta')} + t\eps (e^\eps -1)$ for every $\delta' > 0$.
\end{theorem}

\begin{theorem}[Exponential Mechanism \citep{mcsherry2007mechanism}]\label{thm:exp-mech}
    Let $\eps,\Delta>0$, $\beta \in (0,1)$ be parameters. Let $\cU$ be an arbitrary domain and $\cY$ be a finite nonempty domain, and $S\in \cU^*$ be an arbitrary dataset. Let $\cQ: \cU^* \times \cY \rightarrow \mathbb{R}$ be a quality function with sensitivity $\Delta$. Then there is an algorithm, called the Exponential Mechanism, that is $(\varepsilon, 0)$-DP with the following property. For every $S \in \cU^*$, it outputs an element $\hat{y} \in \cY$ satisfying
    \begin{equation*}
        \cQ(S,\hat{y}) \geq \max_{y \in \cY} \cQ(S,y) - \frac{2\Delta \log(|\cY|/\beta)}{\eps}
    \end{equation*}
    with probability at least $1 - \beta$.
\end{theorem}

\begin{definition}[Truncated Laplace distribution]
\label{def:truncated-laplace}
Let $s,\eps>0$ and $\delta\in(0,1/2)$, and set
\mbox{$R=\frac{s}{\eps}\log\left(1+\frac{e^{\eps}-1}{2\delta}\right)$}.
Then, the Truncated Laplace distribution is denoted by $\operatorname{TLap}(s,\eps,\delta)$ with the following density:
\[
    p(z)=
    \begin{cases}
        \displaystyle
        \frac{\eps}
        {2s(1-e^{-\eps R/s})}
        e^{-\eps|z|/s},
        & |z|\leq R,\\
        0,&\text{otherwise}.
    \end{cases}
\]
\end{definition}

\begin{lemma}[Truncated Laplace mechanism,
Theorem 1 of \citet{geng2020tight}]
\label{lem:truncated-laplace}
Let $\delta\in (0,1)$, $\eps>0$, and $q:\cU^*\to\bR$ have sensitivity at most $s$.
If $Z\sim\operatorname{TLap}(s,\eps,\delta)$,
then $q(S)+Z$ is $(\eps,\delta)$-DP.
Moreover, $|Z|\leq R$ almost surely.
\end{lemma}

\begin{theorem}[GAP-MAX, additive-error form of Theorem IV.6. of \cite{bun2021private}] \label{thm:gap-max}
Let $\eps>0,\beta,\delta\in(0,1)$ and $t\geq 1$ be public parameters, let $\cU$ be an arbitrary domain, and let $\cY$ be a finite nonempty domain. Let $\score\colon \cU^* \times \cY \to \bR$ be a sensitivity-1 function with respect to its first argument. Then there is an algorithm, called the GAP-MAX algorithm, that is $(\eps,\delta)$-DP with the following property. For every $S \in \cU^*$ and $\Delta>0$, if
\begin{equation*}
    \left|\left\{y \in \cY: \score(S,y) \geq \sup_{y \in \cY} \score(S,y) - 5\Delta\right\}\right| \leq t
\end{equation*}
then the output of the algorithm $\hat{y}=\text{GAP-MAX}(\cY,S,\score,\Delta,\beta,\eps,\delta,t)$ satisfies:
\begin{align*}
\prob{\score(S,\hat{y}) \geq \sup_{y \in \cY}\score(S,y) - \Delta} \geq 1-\beta
\end{align*}
provided $\Delta = \Omega\left(1+\frac{\min\{\log|\cY|,\log(1/\delta)\}+\log(t/\beta)}{\eps}\right)$. 
\end{theorem}

\subsection{Convex analysis}
\begin{theorem}[Helly's theorem~\cite{helly1923mengen}]\label{thm:helly}
    Let $\{S_i\}_{i\in \cI} \subseteq \bR^d$ be a finite family of convex sets. If for every subset of indices $\cJ \subseteq \cI$ of size at most $|\cJ|\leq d+1$ we have a non-empty intersection $\bigcap_{j \in \cJ} S_j \neq \emptyset$, then the intersection of all of the sets is also non-empty $\bigcap_{i \in \cI} S_i \neq \emptyset$.
\end{theorem}

\begin{theorem}[Center Point Theorem~\cite{rado1946measure}]\label{thm:center-point-org}
Let $S\in(\bR^d)^*$ be a nonempty dataset.
Then there exists a point $y\in\CH(S)$ such that
\[
    \TD(S,y)\geq\frac{|S|}{d+1}.
\]
Equivalently, every closed halfspace containing $y$
contains at least $|S|/(d+1)$ points of $S$.
\end{theorem}

\section{Omitted proofs from section~\ref{sec:intro-black-box}}
\begin{claim}\label{proof:claim:g-monotone}
    Proof of Claim~\ref{claim:g-monotone}.
\end{claim}
\begin{proof} 
    Note that, if $S_1\subseteq S_2$, then for every $h\in[|S|]$ we have $M^{\downarrow}_{h}(S_1)\subseteq M^{\downarrow}_{h}(S_2)$ by definition, and $\left(\bigcup_{S'\in M^{\downarrow}_{h}(S_1)} f(S')\right) \subseteq \left(\bigcup_{S'\in M^{\downarrow}_{h}(S_2)} f(S')\right)$. Therefore, the convex hull of the left-hand side is contained in the convex hull of the right-hand side.
\end{proof}

\begin{claim}\label{proof:claim:generic-opt-black-box}
    Proof of Claim~\ref{claim:generic-opt-black-box}.
\end{claim}
\begin{proof}
By the definition of a $(\lambda,\beta)$-black-box evaluator, with probability at least $1-\beta$, we have:
\begin{align*}
\hat{y}:=\cA(f,S)\in\CH\left(f\left(N^{\downarrow}_{\lambda}(S)\right)\right).
\end{align*}
By the definition of the convex hull, there exist $S_1,\cdots,S_k\in N^{\downarrow}_{\lambda}(S)$ and coefficients $\lambda_1,\cdots,\lambda_k\geq 0$ such that $\sum_{i\in[k]}\lambda_i=1$ and
$\hat{y}=\sum_{i\in[k]}\lambda_i f(S_i)$.
Since $\cQ(S,.)$ is quasi-concave, we have:
\begin{align*}
\cQ(S,\hat{y})
&=\cQ\left(S,\sum_{i\in[k]}\lambda_i f(S_i)\right)\\
&\geq\min_{i\in[k]}\cQ(S,f(S_i))\\
&\geq\min_{S'\in N^{\downarrow}_{\lambda}(S)}\cQ(S,f(S')).
\end{align*}
Similarly, since $\cL(S,.)$ is convex, we have:
\begin{align*}
\cL(S,\hat{y})
&=\cL\left(S,\sum_{i\in[k]}\lambda_i f(S_i)\right)\\
&\leq\sum_{i\in[k]}\lambda_i\cL(S,f(S_i))\\
&\leq\max_{i\in[k]}\cL(S,f(S_i))\\
&\leq\max_{S'\in N^{\downarrow}_{\lambda}(S)}\cL(S,f(S')).
\end{align*}
\end{proof}

\section{Omitted proofs from section~\ref{sec:black-box-eval-qco}}\label{proofs:sec:black-box-eval-qco}

\begin{claim}\label{proof:claim:depth-monotone-h}
    Proof of Claim~\ref{claim:depth-monotone-h}.
\end{claim}
\begin{proof}
For every dataset $T\in\cU^*$, we have
$M^{\downarrow}_{h'}(T)\subseteq M^{\downarrow}_h(T)$,
and hence $\CDI_{h'}^f(T)\subseteq\CDI_h^f(T)$.
Therefore, every $S'\subseteq S$ satisfying
$y\notin\CDI_h^f(S\setminus S')$ also satisfies
$y\notin\CDI_{h'}^f(S\setminus S')$.
Taking the minimum of $|S'|$ over such deletions gives
$\depth_{h'}^f(S,y)\leq\depth_h^f(S,y)$.
\end{proof}

\begin{claim}\label{proof:claim:depth-quasi}
    Proof of Claim~\ref{claim:depth-quasi}.
\end{claim}
\begin{proof}
        Let's assume by contradiction that there are $y_1,y_2\in \bR^d$, and $\lambda \in [0,1]$ such that $t^* = \depth^f_h\left(S,\lambda y_1 + (1-\lambda) y_2\right) < \min\left\{\depth^f_h\left(S,y_1\right),\depth^f_h\left(S, y_2\right)\right\}$. This means that there exists a subset of points $S' \subseteq S$ of size $|S'|=t^*$ such that the point $\lambda y_1 + (1-\lambda) y_2$ would not be in $\CDI^f_h(S\setminus S')$. Since $t^* < \min\left\{\depth^f_h\left(S,y_1\right),\depth^f_h\left(S, y_2\right)\right\}$, we should still have $y_1,y_2\in \CDI^f_h(S\setminus S')$. Since $\CDI^f_h(S\setminus S')$ is a convex set, it should contain any linear combination of $y_1$ and $y_2$ which is a contradiction.
    \end{proof}

\begin{claim}\label{proof:claim:depth-sens}
   Proof of Claim~\ref{claim:depth-sens}. 
\end{claim}
 \begin{proof}
         We first show that for every $y\in \bR^d$ and $S_1,S_2\subseteq \bR^d$ such that $S_1\subseteq S_2$, we have 
         \begin{enumerate}[label=\Roman*.]
             \item $\depth^f_h(S_1,y)\leq \depth^f_h(S_2,y)$
             \item $\depth^f_h(S_2,y)\leq \depth^f_h(S_1,y) + |S_2\setminus S_1|$
        \end{enumerate}
         
         To prove (I), let $\Bar{S}$ be a set of size $\depth^f_h(S_2,y)$ such that $y \notin \CDI^f_h(S_2\setminus \Bar{S})$. Since $(S_1\setminus \Bar{S}) \subseteq (S_2\setminus \Bar{S})$, using the monotonicity property of function $\CDI^f_h$ (Claim~\ref{claim:g-monotone}), we can say $\CDI^f_h(S_1\setminus \Bar{S}) \subseteq \CDI^f_h(S_2\setminus \Bar{S})$ which implies $y$ cannot be in $\CDI^f_h(S_1\setminus \Bar{S})$ either. Therefore, $\depth^f_h(S_1,y)\leq |\Bar{S}| = \depth^f_h(S_2,y)$.
         
        To prove (II), let $\Tilde{S}$ be a set of size $\depth^f_h(S_1,y)$ such that $y \notin \CDI^f_h(S_1\setminus \Tilde{S})$. 
        Note that $S_1\setminus \Tilde{S} = S_2 \setminus \left((S_2\setminus S_1)\cup \Tilde{S}\right)$, therefore $y\notin \CDI^f_h\left(S_2 \setminus \left((S_2\setminus S_1)\cup \Tilde{S}\right)\right)$, which means $\depth^f_h(S_2,y)\leq |(S_2\setminus S_1)\cup \Tilde{S}| = |S_2\setminus S_1|+ |\Tilde{S}| = |S_2\setminus S_1|+ \depth^f_h(S_1,y)$.
        
        Putting (I) and (II) together gives $\depth^f_h(S_1,y) \leq \depth^f_h(S_2,y) \leq \depth^f_h(S_1,y)+|S_2\setminus S_1|$ when $S_1 \subseteq S_2$.
        Now let $S,S'\subseteq \bR^d$ be two arbitrary sets. We have $(S\cap S') \subseteq S$, which implies $\depth^f_h(S\cap S',y) \leq \depth^f_h(S,y) \leq \depth^f_h(S\cap S',y)+|S\setminus (S\cap S')| = \depth^f_h(S\cap S',y)+|S\setminus S'|$. Similarly, $\depth^f_h(S\cap S',y) \leq \depth^f_h(S',y) \leq \depth^f_h(S\cap S',y)+|S'\setminus (S\cap S')| = \depth^f_h(S\cap S',y)+|S'\setminus S|$. Combining together, concludes the claim.
    \end{proof}

\begin{lemma}\label{proof:lemma:ext1}
    Proof of Lemma~\ref{lemma:ext1}.
\end{lemma}
\begin{proof}
        We construct $\Tilde{\cY}$ as follows.
For every nonempty family $\cI\subseteq 2^{\cY}$
such that $\cA_{\cI}\neq\emptyset$, add an arbitrary
member of $\cA_{\cI}$ to $\Tilde{\cY}$.
Every selected point belongs to $\CH(\cY)$.
By construction, every nonempty family with
$\cA_{\cI}\neq\emptyset$ satisfies
$\cA_{\cI}\cap\Tilde{\cY}\neq\emptyset$.
Moreover, taking $\cI=\{\{y\}\}$ for each $y\in\cY$
shows that $\cY\subseteq\Tilde{\cY}$.
Thus, $\Tilde{\cY}$ is nonempty, which also handles
the empty family $\cI=\emptyset$.
Finally, the number of selected points is at most
$2^{2^{|\cY|}}$.
\end{proof}

\begin{lemma}\label{proof:lemma:ext2}
    Proof of Lemma~\ref{lemma:ext2}.
\end{lemma}

\begin{proof}
    For a nonempty finite set $T\subseteq\bR^d$, define its
    affine hull by
    $
        \operatorname{aff}(T)
        =
        \left\{
            \sum_{y\in T}\alpha_y y :
            \alpha_y\in\bR,\quad \sum_{y\in T}\alpha_y=1
        \right\}.
    $

    Let $n=|\cY|$, and 
    $
        \mathcal L
        =
        \left\{
            \operatorname{aff}\left(\{y_1,\ldots,y_d\}\right) :
            (y_1,\ldots,y_d)\in\cY^d
        \right\}$,
    where repetitions are allowed. Define the extension as:
    \[
        \Tilde{\cY}
        =
        \left\{
            x\in \CH(\cY) :
            \{x\}=\bigcap_{j=1}^d L_j
            \text{ for some }L_1,\ldots,L_d\in\mathcal L
        \right\}.
    \]
    Each ordered $d$-tuple of members of $\mathcal L$
    contributes at most one point, so
    $|\Tilde{\cY}|
        \le |\mathcal L|^d
        \le (n^d)^d
        = n^{d^2}$.

    Now, we show the extension property. Since
    $\operatorname{aff}(\{y,\ldots,y\})=\{y\}$, we have
    $\cY\subseteq\Tilde{\cY}$, which handles the empty family.
    Now let $\cI\subseteq 2^{\cY}$ be nonempty and suppose
    $\cA_{\cI}\neq\emptyset$.
    Each $\CH(S_i)$ is compact, so $\cA_{\cI}$ is compact.
    Let $x^*\in\argmax_{x\in \cA_{\cI}} \|x\|^2$.

    For each $S_i\in\cI$, there exists  a convex representation of $x^*$
    with non-zero coefficients, so we can write:
    \[
        x^*=\sum_{y\in T_i}\lambda_{i,y}y,
        \qquad
        T_i\subseteq S_i,\quad
        \lambda_{i,y}>0,\quad
        \sum_{y\in T_i}\lambda_{i,y}=1.
    \]
    Let $M_i=\operatorname{aff}(T_i)$.
    We will show that $\bigcap_{S_i\in\cI}M_i=\{x^*\}$. 
    Trivially, we have $x^* \in \bigcap_{S_i\in\cI}M_i$. It's enough to show that there is no point other than $x^*$ in the intersection.
    By contradiction, choose a nonzero vector $v$ such that $x^*+v$
    belongs to every $M_i$. By the definition of affine hull, there are coefficients $\alpha_{i,y}$ satisfying
    $x^*+v=\sum_{y\in T_i}\alpha_{i,y}y$, and $\sum_{y\in T_i}\alpha_{i,y}=1$.

Now, let $\mu_{i,y}=\alpha_{i,y}-\lambda_{i,y}$, and observe that $v=\sum_{y\in T_i}\mu_{i,y}y$, and $\sum_{y\in T_i}\mu_{i,y}=0$.
    
    All $\lambda_{i,y}$ are strictly positive and there are only
    finitely many coefficients. Hence we can choose
    $\eps>0$ small enough that
    $\lambda_{i,y}\pm\eps\mu_{i,y}\ge0$ for every $i,y$.
    As a result, for every $i$ we have:
    \[
        x^*\pm\eps v
        =
        \sum_{y\in T_i}
        (\lambda_{i,y}\pm\eps\mu_{i,y})y
        \in\CH(T_i)\subseteq\CH(S_i)
    \]
    Therefore, both $x^*\pm\eps v$  belong to
    $\cA_{\cI}$, contradicting $x^*\in\argmax_{x\in \cA_{\cI}} \|x\|^2$. This is because:
    \[
        \frac{\|x^*+\eps v\|^2+
                    \|x^*-\eps v\|^2}{2}
        =\|x^*\|^2+\eps^2\|v\|^2
        >\|x^*\|^2.
    \]

Thus, we showed that $\bigcap_{S_i\in\cI}M_i=\{x^*\}$. It remains to show that $x^*\in\Tilde{\cY}$.
    By definition, we must express $\{x^*\}$ as the intersection
    of exactly $d$ members of $\mathcal L$.

    First, we show that there exists a subset of the $M_i$'s of size at most $d$ whose
    intersection is $\{x^*\}$. Start with $V_0=\bR^d$.
    Whenever $V_j\neq\{x^*\}$, choose $z\in V_j\setminus\{x^*\}$.
    Since $\bigcap_{S_i\in\cI}M_i=\{x^*\}$, some $M_i$ does
    not contain $z$. Set
    $V_{j+1}=V_j\cap M_i$.
    
    This intersection contains $x^*$ but excludes $z$, so it
    is a proper affine subspace of $V_j$. Hence
    $\dim(V_{j+1})<\dim(V_j)$.
    Starting from dimension $d$, after at most $d$ steps
    we therefore reach $\{x^*\}$. Denote the selected spaces by
    $M_{i_1},\ldots,M_{i_k}$, where $1\le k\le d$. Thus
    $\{x^*\}=\bigcap_{j=1}^k M_{i_j}$.

    Second, we show that every selected space belongs to
    $\mathcal L$. Fix such a space $M_i=\operatorname{aff}(T_i)$.
    Because it decreased the current intersection, it cannot
    equal $\bR^d$; write $r=\dim(M_i)\le d-1$.
    Choose $y_0\in T_i$. The vectors
    $\{y-y_0:y\in T_i\}$ span the directions of $M_i$,
    so we can choose $y_1,\ldots,y_r\in T_i$ such that
    $y_1-y_0,\ldots,y_r-y_0$ form a basis. Consequently,
    $M_i=\operatorname{aff}(\{y_0,\ldots,y_r\})$.
    This uses $r+1\le d$ points of $\cY$. Repeating points
    if necessary gives a $d$-tuple with the same affine hull,
    proving that $M_i\in\mathcal L$.

    We have expressed $\{x^*\}$ as the intersection of
    $k\le d$ members of $\mathcal L$. If $k<d$, repeat one
    of these members until there are exactly $d$; repetitions
    do not change the intersection. Since $x^*\in\cA_{\cI}\subseteq\CH(\cY)$ and
$\{x^*\}$ is the intersection of exactly $d$ members
of $\mathcal L$, we have $x^*\in\Tilde{\cY}$.
Hence, $\cA_{\cI}\cap\Tilde{\cY}\neq\emptyset$.

\end{proof}

\begin{lemma}\label{proof:thm:good-score}
    Proof of Lemma~\ref{thm:good-score}.
\end{lemma}
\begin{proof}
    Let $\lambda = \frac{|S|-h}{d+1}$. We need to show that there exists a point $y^*\in \Tilde{\cY}$, such that for every set $S'\subseteq S$ of size $|S'|\geq |S|-\lambda$, we have $y^* \in \CDI_h^f(S')$. In other words, it is enough to show that $\left[\bigcap_{S'\in N^{\downarrow}_{\lambda}(S)}\CDI_h^f(S')\right] \cap \Tilde{\cY}\neq \emptyset$.

Let $S_1,\cdots, S_{d+1}\in N^{\downarrow}_{\lambda}(S)$ be $d+1$ sets of size at least $|S|-\lambda$, and $S_0 = \bigcap_{i\in[d+1]}S_i$ be their intersection. We have $|S_0|\geq |S|-\sum_{i\in[d+1]}|S\setminus S_i|= |S|-(d+1)\lambda = h$, which implies that $S_0 \in M^{\downarrow}_{h}(S_0) \neq \emptyset$ and as a result $\CDI_h^f(S_0)\neq \emptyset$. Putting together with the monotonicity property of function $\CDI_h^f(.)$ (Lemma~\ref{claim:g-monotone}), we have $\CDI_h^f(S_0) \subseteq \CDI_h^f(S_i)$ for every $i\in[d+1]$, which implies $\emptyset \neq \CDI_h^f(S_0)\subseteq \bigcap_{i\in[d+1]}\CDI_h^f(S_i)$. Now, by Helly's theorem~\ref{thm:helly}, we can say that $\bigcap_{S'\in N^{\downarrow}_{\lambda}(S)}\CDI_h^f(S')\neq \emptyset$.

Furthermore, we can write:
\begin{align*}
\bigcap_{S'\in N^{\downarrow}_{\lambda}(S)}\CDI_h^f(S') &= \bigcap_{S'\in N^{\downarrow}_{\lambda}(S)}\CH(f(M^{\downarrow}_{h}(S')))\\
&= \bigcap_{ \substack{S' \subseteq S \\ \text{s.t. } |S'| = \left\lceil |S| - \lambda \right\rceil} } \CH(f(M^{\downarrow}_{h}(S')))
\end{align*}

Now, we use the definition~\ref{def:extension}, and set $\cI=
    \left\{
        f(M^{\downarrow}_h(S')):
        S'\in N^{\downarrow}_{\lambda}(S)
    \right\}
    \subseteq 2^{\cY}$, and $\cA_{\cI}=\bigcap_{\cT\in \cI} \CH(\cT)$. Since $\cA_{\cI}=\bigcap_{S'\in N^{\downarrow}_{\lambda}(S)}\CDI_h^f(S') \neq \emptyset$, it concludes that $\left[\bigcap_{S'\in N^{\downarrow}_{\lambda}(S)}\CDI_h^f(S')\right] \cap \Tilde{\cY}\neq \emptyset$.
     \end{proof}

\begin{theorem}\label{proof:thm:main-eval-approx}
    Proof of Theorem~\ref{thm:main-eval-approx}.
\end{theorem}
\begin{proof}
For $|\cY|=1$, returning its unique point gives the result.
Otherwise, use Algorithm~\ref{alg:noisy-cutoff-eval}. We first show the privacy guarantee. The dataset-size query has sensitivity
$1$, so Lemma~\ref{lem:truncated-laplace} implies that
$\tilde n$ is $(\eps_0,\delta_0)$-DP.
For every fixed value of $\tilde n$, the cutoff $h$ is fixed.
If $h\leq0$, the output $f(\emptyset)$ is independent of $S$.
If $h\geq1$, Theorem~\ref{thm:main} gives
$(\eps_0,\delta_0)$-DP for the second stage.
Adaptive composition therefore gives
$(2\eps_0,2\delta_0)=(\eps,\delta)$-DP.

We next prove the utility. Let $n=|S|$, and $\lambda=\left\lceil K+2R+1\right\rceil$.
Since $|Z|\leq R$ (Lemma~\ref{lem:truncated-laplace}), whenever $h\geq1$, we have
$K\leq n-h<K+2R+1\leq\lambda$.
Every query used to compute $\depth_h^f(T,y)$, with
$T\subseteq S$, evaluates $f$ on a subset $U\subseteq T$
with $|U|\geq h$.
Thus $|S\setminus U|\leq n-h\leq\lambda$.
If $|T|<h$, the score is zero and no oracle query is needed.

If $h\leq0$, then $n<K+R-Z+1\leq K+2R+1\leq\lambda$,
so $\emptyset\in N^\downarrow_\lambda(S)$.
Consequently, every oracle query is permitted on every run.

Finally, suppose $h\geq1$. Conditional on the private cutoff,
Theorem~\ref{thm:main} gives:
\[
    \depth_h^f(S,\hat y)
    \geq\frac{n-h}{d+1}-\Delta
    \geq\frac{K}{d+1}-\Delta
    \geq1,
\]
with probability at least
$1-\beta_0$. Positive depth implies $\hat y\in\CDI_h^f(S)
    \subseteq
    \CH\left(f(N^\downarrow_\lambda(S))\right)$.
    
The inclusion follows from $n-h\leq\lambda$.
In the case of $h\leq0$, the output $f(\emptyset)$
satisfies the same utility guarantee deterministically.
The overall failure probability is therefore at most
$\beta_0\leq\beta$.

Lemma~\ref{lemma:ext1} bounds the extension size $|\tilde{\cY}|$.
Moreover, $\lambda=O((d+1)(\Delta+1)+R)$, and $R=O\left(\frac{\log(1/\delta)}{\eps}\right)$.
Substituting the QCO error bound gives the stated asymptotic bound.
\end{proof}

\begin{theorem}\label{proof:thm:main-eval-pure}
    Proof of
Theorem~\ref{thm:main-eval-pure}.
\end{theorem}
\begin{proof}
For $|\cY|=1$, returning its unique point gives the result. Otherwise, use Algorithm~\ref{alg:noisy-cutoff-eval-pure}.

We first show the privacy guarantee.
Let $\cB$ denote the mechanism obtained by omitting
the uniform-output branch and running the remaining
steps of the algorithm.
The dataset-size query has sensitivity $1$, so
Lemma~\ref{lem:truncated-laplace} implies that
computing $\tilde n$ is $(\eps_0,\delta_0)$-DP.
For every fixed value of $\tilde n$, the cutoff $h$
is fixed.
If $h\leq0$, the output $f(\emptyset)$ is independent
of $S$.
If $h\geq1$, Theorem~\ref{thm:main-pure} gives
$(\eps_0,0)$-DP for the second stage.
Privacy composition therefore implies that
$\cB$ is $(2\eps_0,\delta_0)
=(\eps,\delta_0)$-DP.

The output of $\cB$ belongs to $\tilde{\cY}$.
Indeed, the pure depth optimizer outputs in
$\tilde{\cY}$, and the extension property implies
$\cY\subseteq\tilde{\cY}$, so
$f(\emptyset)\in\tilde{\cY}$ as well.
For neighboring datasets $S,S'$ and every
$y\in\tilde{\cY}$, write
$p_S(y)=\Pr[\cB(S)=y]$.
The preceding privacy guarantee gives $p_S(y)\leq e^\eps p_{S'}(y)+\delta_0$.

Let $\cA$ denote Algorithm~\ref{alg:noisy-cutoff-eval-pure}.
Since $\Pr[\cA(S)=y]=(1-\eta)p_S(y)+\frac{\eta}{M}$,
we get
\begin{align*}
    \Pr[\cA(S)=y]
    -e^\eps\Pr[\cA(S')=y]
    &=
    (1-\eta)
    \left(p_S(y)-e^\eps p_{S'}(y)\right)
    -\frac{\eta(e^\eps-1)}{M}\\
    &\leq
    (1-\eta)\delta_0
    -\frac{\eta(e^\eps-1)}{M}\leq0.
\end{align*}
The last inequality follows from
$\delta_0\leq\eta(e^{\eps_0}-1)/M$
and $\eps_0\leq\eps$.

We next verify the oracle-query restriction.
Let $n=|S|$, and set $\lambda=\left\lceil K+2R+1\right\rceil$.
The uniform-output branch makes no oracle queries.
In the remaining branch, Lemma~\ref{lem:truncated-laplace}
gives $|Z|\leq R$ almost surely.
Thus, whenever $h\geq1$, we have $K\leq n-h<K+2R+1\leq\lambda$.
Every query used to compute $\depth_h^f(T,y)$,
with $T\subseteq S$, evaluates $f$ on a subset
$U\subseteq T$ with $|U|\geq h$.
Consequently,
$|S\setminus U|\leq n-h\leq\lambda$.
If $|T|<h$, the score is zero and no oracle query
is needed.

If $h\leq0$, then $n<K+R-Z+1\leq K+2R+1\leq\lambda$,
so $\emptyset\in N^\downarrow_\lambda(S)$.
Therefore, every oracle query is permitted
on every run.

Finally, we prove the utility.
Conditional on taking the nonuniform branch,
if $h\leq0$, the output $f(\emptyset)$ satisfies
the required convex-hull guarantee deterministically.
If $h\geq1$, then, conditional on the private cutoff,
Theorem~\ref{thm:main-pure} gives
\[
    \depth_h^f(S,\hat y)
    \geq\frac{n-h}{d+1}-\Delta
    \geq\frac{K}{d+1}-\Delta
    \geq1
\]
with probability at least $1-\beta_0$.
Positive depth implies $\hat y\in\CDI_h^f(S)
    \subseteq
    \CH\left(f(N^\downarrow_\lambda(S))\right)$,
where the inclusion follows from $n-h\leq\lambda$.
The overall failure probability is therefore at most
$\eta+(1-\eta)\beta_0\leq\beta$.

It remains to bound $\lambda$.
If
$\delta_0=\eta(e^{\eps_0}-1)/M$, then $R=\frac{1}{\eps_0}
    \log\left(1+\frac{M}{2\eta}\right)
    =\frac{2}{\eps}
    \log\left(1+\frac{M}{\beta}\right)$.
Otherwise, $\delta_0=1/4$, and $R=\frac{1}{\eps_0}
    \log(2e^{\eps_0}-1)
    \leq1+\frac{\log2}{\eps_0}$.
Consequently, $R=O\left(
        1+\frac{\log(M/\beta)}{\eps}
    \right)$.
Together with the definitions of $\Delta$ and $K$,
this gives $\lambda
    =O\left(
        d+\frac{d\log(M/\beta)}{\eps}
    \right)$.
Lemma~\ref{lemma:ext2} gives
$M\leq|\cY|^{d^2}$, and hence $\lambda
    =O\left(
        d+\frac{d^3\log(|\cY|/\beta)}{\eps}
    \right)$.
For $0<\eps\leq1$, the additive $d$ term is absorbed,
giving the stated bound.
\end{proof}

\section{Omitted proofs from section~\ref{sec:black-box-qco-equi}}

\begin{theorem}\label{proof:thm:black-box-qco-equi}
    Proof of Theorem~\ref{thm:black-box-qco-equi}.
\end{theorem}
\begin{proof}
Fix an input domain $\cU$ and a nonempty finite
domain $\cY\subseteq\bR^d$ satisfying
$\log^*|\cY|\leq L$.

We first prove the lower bound.
Let $N=n_{\mathrm{black\text{-}box}}
    (\alpha,\beta,\eps,\delta,d,L)$.
The inequality is immediate if $N=\infty$.
Otherwise, let $\cA$ be a corresponding black-box
evaluator on $\cY$, and set
$\lambda=\lfloor\alpha N\rfloor$.

Given a discrete quasi-concave quality function
$\cQ:\cU^*\times\cY\to\bR$ with sensitivity at most $1$,
define
\[
    f_{\cQ}(T)\in
    \argmax_{y\in\cY}\cQ(T,y),
\]
breaking ties according to a fixed public ordering
of $\cY$.
Run $\cA(f_{\cQ},S)$, replacing any output outside
$\CH(\cY)$ by a fixed public point of $\cY$.

For every fixed $\cQ$, the function $f_{\cQ}$ is fixed
on all datasets, so privacy follows from the guarantee
of $\cA$ and post-processing.
Each query $f_{\cQ}(T)$ is answered by evaluating
$\cQ(T,y)$ over $y\in\cY$.
Since every queried $T$ belongs to
$N^\downarrow_\lambda(S)$, the resulting algorithm
satisfies the QCO oracle-access restriction
on every input.

Suppose $|S|\geq N$.
For every $T\in N^\downarrow_\lambda(S)$,
bounded sensitivity gives
\begin{align*}
    \cQ(S,f_{\cQ}(T))
    &\geq
    \cQ(T,f_{\cQ}(T))-\lambda\\
    &=
    \max_{y\in\cY}\cQ(T,y)-\lambda\\
    &\geq
    \max_{y\in\cY}\cQ(S,y)-2\lambda.
\end{align*}
With probability at least $1-\beta$, the evaluator
outputs $\hat y\in
    \CH\left(f_{\cQ}(N^\downarrow_\lambda(S))\right)$.
The canonical-extension definition
(Remark~\ref{rem:quasi-func-ext}) gives
\begin{align*}
    \cQ(S,\hat y)
    &\geq
    \min_{T\in N^\downarrow_\lambda(S)}
    \cQ(S,f_{\cQ}(T))\\
    &\geq
    \max_{y\in\cY}\cQ(S,y)-2\lambda\\
    &\geq
    \max_{y\in\cY}\cQ(S,y)-2\alpha N.
\end{align*}
Thus, one algorithm satisfies the QCO guarantee
with accuracy parameter $2\alpha$ on every dataset
of size at least $N$, proving the lower bound.

For the upper bound, let $N=
    \max\left\{
        n_{\mathrm{QCO}}
        \left(
            \frac{\alpha}{2(d+1)},
            \beta,\eps/2,\delta/2,d,L+2
        \right),
        \left\lceil
            \frac{C}{\alpha}
            \left(
                d+\frac{\log(1/\delta)}{\eps}
            \right)
        \right\rceil
    \right\}$.
If $N=\infty$, the result is immediate.
Otherwise, the proof is similar to the analysis of
Algorithm~\ref{alg:noisy-cutoff-eval}, using the QCO
algorithm witnessing $n_{\mathrm{QCO}}
    \left(
        \frac{\alpha}{2(d+1)},
        \beta,\eps/2,\delta/2,d,L+2
    \right)$
to optimize the CDI depth function $\depth_h^f$
after choosing the cutoff $h$, with additive error
at most
$\Delta=\alpha N/(2(d+1))$ and failure parameter $\beta$.
For sufficiently large $C$, the same argument gives
$(\eps,\delta)$-DP and the oracle-access restriction
with radius $\lfloor\alpha N\rfloor$ on every dataset,
and the required utility guarantee on every dataset
of size at least $N$.
This proves the upper bound.
\end{proof}

\section{Omitted proofs from section~\ref{sec:applications}}
\begin{lemma}\label{proof:lemma:depth-td}
    Proof of Lemma~\ref{lemma:depth-td}.
\end{lemma}
\begin{proof}
Using the definition, we can write:
    \begin{align*}
        \depth_h^f(S,y) &= \min \left\{|S'|: S' \subseteq S, y\notin \CDI^f_h(S\setminus S')\right\} \\
        &= \min \left\{|S'|: S' \subseteq S, y\notin \CH\left(\bigcup_{S''\subseteq (S\setminus S'), |S''|\geq h} f(S'')\right)\right\}
    \end{align*}
    In the last term, for any $S''\subseteq (S\setminus S')$, it holds that $\CH(S'')\subseteq \CH(S\setminus S')$. Since $f$ is internal, we also have $f(S'')\subseteq \CH(S'')$ for any such $S''$. As a result:
    \begin{align*}
        \bigcup_{S''\subseteq (S\setminus S'), |S''|\geq h} f(S'') \subseteq \CH(S\setminus S')
    \end{align*}
    This also implies $\CH\left(\bigcup_{S''\subseteq (S\setminus S'), |S''|\geq h} f(S'')\right) \subseteq \CH(S\setminus S')$. Putting together, we get:
    \begin{align*}
        \depth_h^f(S,y) 
        &= \min \left\{|S'|: S' \subseteq S, y\notin \CH\left(\bigcup_{S''\subseteq (S\setminus S'), |S''|\geq h} f(S'')\right)\right\}\\
        &\leq \min \left\{|S'|: S' \subseteq S, y\notin \CH\left(S\setminus S'\right)\right\} = TD(S,y)
    \end{align*}

    For the case where $h=1$, it is easy to observe the other direction:
    \begin{align*}
        \CH\left(\bigcup_{S''\subseteq (S\setminus S'), |S''|\geq 1} f(S'')\right) &\supseteq \CH\left(\bigcup_{S''\in (S\setminus S')} f(S'')\right)\\
        &= \CH\left(\bigcup_{S''\in (S\setminus S')} S''\right)\\
        &= \CH\left( S\setminus S'\right) 
    \end{align*}
    where the first inequality used the fact that the output of an "internal" function on a single element is that element itself.
\end{proof}

\begin{lemma}\label{proof:lemma:tm-improved}
    Proof of Corollary~\ref{lemma:tm-improved}.
\end{lemma}
\begin{proof}
Run Algorithm~\ref{alg:main} with $h=1$, input and
output domain $\cY$, and the extension $\tilde{\cY}$
from Lemma~\ref{lemma:ext1}. The privacy guarantee
follows from Theorem~\ref{thm:main}.
Since $\log^*|\tilde{\cY}|\leq\log^*|\cY|+2$,
the same theorem gives the stated error bound.

Furthermore, $f_{\operatorname{center}}$ is internal
on nonempty datasets, so Lemma~\ref{lemma:depth-td}
gives
$\depth_1^{f_{\operatorname{center}}}(S,y)=\TD(S,y)$.
Consequently, with probability at least $1-\beta$,
\begin{align*}
    \TD(S,\hat{y})
    &=
    \depth_1^{f_{\operatorname{center}}}(S,\hat{y})\\
    &\geq
    \max_{y\in\tilde{\cY}}
    \depth_1^{f_{\operatorname{center}}}(S,y)-\Delta\\
    &=
    \max_{y\in\tilde{\cY}}\TD(S,y)-\Delta.
\end{align*}

It remains to show that the maximum Tukey depth over
$\bR^d$ is attained at a point in $\tilde{\cY}$.
For every integer $t\geq1$, we have
\[
    \{y\in\bR^d:\TD(S,y)\geq t\}
    =
    \bigcap_{\substack{D\subseteq S\\|D|<t}}
    \CH(S\setminus D).
\]
Each set in this intersection is the convex hull
of a subset of $\cY$. Therefore, the extension
property implies that every nonempty such superlevel
set intersects $\tilde{\cY}$.
Since Tukey depth takes values in
$\{0,1,\ldots,|S|\}$, its maximum is attained,
and the deepest superlevel set is nonempty.
Thus,
\[
    \max_{y\in\tilde{\cY}}\TD(S,y)
    =
    \max_{y\in\bR^d}\TD(S,y).
\]
Finally, the center point theorem gives
$\max_{y\in\bR^d}\TD(S,y)\geq |S|/(d+1)$,
which concludes the proof.
\end{proof}

\begin{corollary}\label{proof:cor:user-level-center}
    Proof of Corollary~\ref{cor:user-level-center}.
\end{corollary}
\begin{proof}
We treat each user's dataset as a single element of the
input domain $\cY^*$.
The privacy is due to Theorem~\ref{thm:main-eval-approx}.
Moreover, Lemma~\ref{lemma:ext1} gives
$\log^*|\tilde{\cY}|\leq\log^*|\cY|+2$,
so the evaluator's deletion parameter $\lambda$
has the stated asymptotic bound.

We first verify that $f_{\operatorname{center}}$ attains
the globally optimal Tukey depth.
For any users' dataset $T$, write
$P=\operatorname{flat}(T)$, counting multiplicities.
For every positive integer $t$,
\[
    \{y\in\bR^d:\TD(P,y)\geq t\}
    =
    \bigcap_{\substack{D\subseteq P\\|D|<t}}
    \CH(P\setminus D).
\]
Each convex hull on the right is the convex hull of
a subset of $\cY$. Therefore, the extension property
ensures that every nonempty positive-depth superlevel
set intersects $\tilde{\cY}$.
In particular, when $P$ is nonempty, there exists
$y^*\in\tilde{\cY}$ such that
$\TD(P,y^*)=\max_{y\in\bR^d}\TD(P,y)$.
Consequently, by the definition of
$f_{\operatorname{center}}$ and the center point theorem,
\[
    \TD\left(P,f_{\operatorname{center}}(T)\right)
    =
    \max_{y\in\tilde{\cY}}\TD(P,y)
    =
    \max_{y\in\bR^d}\TD(P,y)
    \geq \frac{|P|}{d+1}.
\]
For $P=\emptyset$, the same statement holds trivially.

Assume now that $n\geq\lambda$.
Let $\cQ(S,y)=\TD(\operatorname{flat}(S),y)$.
Since Tukey depth is quasi-concave in its second argument,
Corollary~\ref{cor:user-level-quasi-opt}, applied with
output domain $\tilde{\cY}$, gives
\[
    \TD\left(\operatorname{flat}(S),\hat{y}\right)
    \geq
    \min_{S'\in N^{\downarrow}_{\lambda}(S)}
    \TD\left(
        \operatorname{flat}(S),
        f_{\operatorname{center}}(S')
    \right)
\]
with probability at least $1-\beta$.

Fix any $S'\in N^{\downarrow}_{\lambda}(S)$, and define $M(S\setminus S')
    :=\sum_{S_i\in S\setminus S'}m_i$.
Since $S'$ is obtained by deleting at most $\lambda$ users, we have $M(S\setminus S')
    \leq \sum_{j\in[\lambda]}m_{i_j}$.

Adding points cannot decrease the Tukey depth of a fixed
point, and deleting $M(S\setminus S')$ points can decrease
it by at most $M(S\setminus S')$.
Together with the global optimality established above,
this gives
\begin{align*}
    \TD\left(
        \operatorname{flat}(S),
        f_{\operatorname{center}}(S')
    \right)
    &\geq
    \TD\left(
        \operatorname{flat}(S'),
        f_{\operatorname{center}}(S')
    \right)\\
    &=\max_{y\in\bR^d}
    \TD\left(\operatorname{flat}(S'),y\right)\\
    &\geq\max_{y\in\bR^d}
    \TD\left(\operatorname{flat}(S),y\right)
    -M(S\setminus S')\\
    &\geq\max_{y\in\bR^d}
    \TD\left(\operatorname{flat}(S),y\right)
    -\sum_{j\in[\lambda]}m_{i_j}.
\end{align*}
Furthermore, the center point guarantee gives
\begin{align*}
    \TD\left(
        \operatorname{flat}(S),
        f_{\operatorname{center}}(S')
    \right)
    &\geq
    \TD\left(
        \operatorname{flat}(S'),
        f_{\operatorname{center}}(S')
    \right)\\
    &\geq \frac{|\operatorname{flat}(S')|}{d+1}\\
    &=\frac{M-M(S\setminus S')}{d+1}\\
    &\geq\frac{M-\sum_{j\in[\lambda]}m_{i_j}}{d+1}.
\end{align*}
Both lower bounds hold for every
$S'\in N^{\downarrow}_{\lambda}(S)$.
Taking their maximum and then the minimum over $S'$
proves the claimed guarantee.
\end{proof}

\begin{corollary}\label{proof:cor:user-level-co}
    Proof of Corollary~\ref{cor:user-level-co}.
\end{corollary}
\begin{proof}
Treat each user's dataset $S_i\in\cX^*$ as a single element
of the input domain. Thus, we set $\cU=\cX^*$ and apply
Claim~\ref{claim:generic-opt-black-box} with the black-box
function $f_{\cL}$.

Since $\cL(S,\cdot)$ is convex,
Claim~\ref{claim:generic-opt-black-box} implies that, with
probability at least $1-\beta$,
\[
\cL(S,\hat{y})
\leq
\max_{S'\in N^{\downarrow}_{\lambda}(S)}
\cL(S,f_{\cL}(S')).
\]

Let $y^*=f_{\cL}(S)$, and fix any
$S'\in N^{\downarrow}_{\lambda}(S)$.
By the optimality of $f_{\cL}(S')$ on $S'$, we have
\[
\cL(S',f_{\cL}(S'))
\leq
\cL(S',y^*).
\]
Consequently,
\begin{align*}
\cL(S,f_{\cL}(S'))
&=
\cL(S',f_{\cL}(S'))
+
\sum_{S_i\in S\setminus S'}
\phi_i(S_i,f_{\cL}(S'))\\
&\leq
\cL(S',y^*)
+
\sum_{S_i\in S\setminus S'}
\phi_i(S_i,f_{\cL}(S'))\\
&=
\cL(S,y^*)
+
\sum_{S_i\in S\setminus S'}
\left[
\phi_i(S_i,f_{\cL}(S'))
-
\phi_i(S_i,y^*)
\right]\\
&=
\min_{y\in\cY}\cL(S,y)
+
\sum_{S_i\in S\setminus S'}
\left[
\phi_i(S_i,f_{\cL}(S'))
-
\phi_i(S_i,f_{\cL}(S))
\right].
\end{align*}
Taking the maximum over all
$S'\in N^{\downarrow}_{\lambda}(S)$ proves the utility guarantee.
The privacy guarantee follows from
Theorem~\ref{thm:main-eval-approx}.
\end{proof}

\section{Iterative log facts}

Here, we collect several elementary facts about iterated logarithms that will be used to establish the utility bounds for our private quasi-concave optimizers. Throughout this section, all logarithms are to base $2$.

\begin{definition}[Iterative floored logarithm]\label{def:iter-log-floor-n}
For $N\geq 0$, let $\log^{\lfloor N \rfloor}(\cdot)$ denote the outcome of $N$ compositions of the function $\lfloor\log(\cdot)\rfloor$, whenever these compositions are defined, with $\log^{\lfloor 0 \rfloor}(n)=n$. That is,
$$
\log^{\lfloor N \rfloor}(n)=
\underbrace{
\lfloor\log\lfloor\log\lfloor\dots\lfloor\log(n)\rfloor\dots\rfloor\rfloor\rfloor
}_{N \text{ times}}.
$$
\end{definition}

\begin{definition}[Iterative floored $\log^*$]
Given an integer $n\geq 1$, let $\log^{\lfloor * \rfloor}(n)$ denote the number of times that the function $\lfloor\log(x)\rfloor$ must be composed before the result is less than or equal to $1$. Formally,
$$
\log^{\lfloor * \rfloor}(n)=
\begin{cases}
1+\log^{\lfloor * \rfloor}\left(\lfloor\log(n)\rfloor\right) & \text{if } n>1,\\
0 & \text{otherwise.}
\end{cases}
$$
\end{definition}

\begin{claim}\label{lem:floor_log_N_bound}
For every integer $n\geq 1$ and every integer $0\leq N \leq \log^*(n)$, we have
$$
\log^{\lfloor N \rfloor}(n)
\leq
\underbrace{\log\log\dots\log(n)}_{N \text{ times}}.
$$
\end{claim}
\begin{proof}
The case $N=0$ is immediate. By Claim~\ref{lem:floor_log_star_bound}, the floored compositions are defined for $N\leq\log^*(n)$: at most one additional composition after reaching $1$ is needed, and it gives $0$. We prove the remaining cases by induction on $N$.

For $N=1$, by definition,
$$
\log^{\lfloor 1 \rfloor}(n)=\lfloor\log(n)\rfloor.
$$
Since $\lfloor x\rfloor \leq x$ for every real number $x$, we have:
$$
\log^{\lfloor 1 \rfloor}(n)=
\lfloor\log(n)\rfloor
\leq
\log(n).
$$

Now, assume the claim holds for some integer $k\geq 1$, namely:
$$
\log^{\lfloor k \rfloor}(n)
\leq
\underbrace{\log\log\dots\log(n)}_{k \text{ times}}.
$$
By definition,
$$
\log^{\lfloor k+1 \rfloor}(n)=
\left\lfloor
\log\left(\log^{\lfloor k \rfloor}(n)\right)
\right\rfloor.
$$
Since both $\log(\cdot)$ and $\lfloor \cdot \rfloor$ are monotonically non-decreasing, we obtain
\begin{align*}
\log^{\lfloor k+1 \rfloor}(n)
&=
\left\lfloor
\log\left(\log^{\lfloor k \rfloor}(n)\right)
\right\rfloor \\
&\leq
\left\lfloor
\log\left(
\underbrace{\log\log\dots\log(n)}_{k \text{ times}}
\right)
\right\rfloor \\
&\leq
\log\left(
\underbrace{\log\log\dots\log(n)}_{k \text{ times}}
\right) \\
&=
\underbrace{\log\log\dots\log(n)}_{k+1 \text{ times}}.
\end{align*}
This completes the induction.
\end{proof}

\begin{claim}\label{lem:floor_log_star_bound}
For every integer $n\geq 1$, we have
$$
\log^{\lfloor * \rfloor}(n)
\leq
\log^*(n)
\leq
\log^{\lfloor * \rfloor}(n)+1.
$$
\end{claim}
\begin{proof}
For $k\in \bN$ define:
\begin{align*}
    \text{Tow}(k)=\begin{cases}
1 & \text{if } k = 0 \\
2^{\text{Tow}(k-1)} & \text{if } k \geq 1.
\end{cases}
\end{align*}
 By definition,
$$
\log^*(n)=\min\left\{k:n\leq \text{Tow}(k)\right\}.
$$

Now, we claim that:
$$
\log^{\lfloor * \rfloor}(n)\leq k
\quad\Longleftrightarrow\quad
n<\text{Tow}(k+1).
$$
We prove this by induction on $k$. For $k=0$, we have $\log^{\lfloor * \rfloor}(n)\leq 0$ if and only if $n\leq 1$. Since $n$ is an integer, this is equivalent to $n<2=\text{Tow}(1)$.

Now, suppose the statement holds for some $k\geq 0$. Then, for $n>1$,
\begin{align*}
\log^{\lfloor * \rfloor}(n)\leq k+1
&\Longleftrightarrow
1+\log^{\lfloor * \rfloor}(\lfloor\log(n)\rfloor)\leq k+1 \\
&\Longleftrightarrow
\log^{\lfloor * \rfloor}(\lfloor\log(n)\rfloor)\leq k \\
&\Longleftrightarrow
\lfloor\log(n)\rfloor<\text{Tow}(k+1) \\
&\Longleftrightarrow
\log(n)<\text{Tow}(k+1) \\
&\Longleftrightarrow
n<2^{\text{Tow}(k+1)} \\
&\Longleftrightarrow
n<\text{Tow}(k+2).
\end{align*}
Therefore,
$$
\log^{\lfloor * \rfloor}(n)=
\min\left\{k:n<\text{Tow}(k+1)\right\}.
$$

Comparing this with
$$
\log^*(n)=\min\{k:n\leq \text{Tow}(k)\},
$$
we immediately get
$$
\log^{\lfloor * \rfloor}(n)
\leq
\log^*(n)
\leq
\log^{\lfloor * \rfloor}(n)+1.
$$
\end{proof}

\section{Agnostic private QCO}\label{sec:agnostic-qco}
In this section, we develop a one-dimensional private quasi-concave optimizer with an agnostic additive-error guarantee: it approximates the optimum without requiring its value or a lower bound as input. We then extend the construction to high dimensions using coordinatewise optimization.

Throughout this section, we sometimes evaluate quality functions on extensions of their original domains. The following remark describes a canonical extension that preserves quasi-concavity and sensitivity.
\begin{remark}[Canonical quasi-concave extension]
\label{rem:quasi-func-ext}
Let $\cY\subseteq\bR^d$ be a non-empty finite domain and
$f:\cY\rightarrow\bR$ be a discrete quasi-concave function.
Write $m=\min_{z\in\cY}f(z)$.
The canonical extension $\bar{f}:\bR^d\rightarrow\bR$
is defined as:
\[
    \bar{f}(x)=
    \begin{cases}
        \displaystyle
        \max_{\substack{\emptyset\neq\cA\subseteq\cY:\\
                       x\in\CH(\cA)}}
        \min_{z\in\cA}f(z)
        & \text{if }x\in\CH(\cY),\\[12pt]
        m & \text{otherwise.}
    \end{cases}
\]
Since $\cY$ is finite, the maximum is attained whenever
$x\in\CH(\cY)$. Discrete quasi-concavity and the choice
$\cA=\{x\}$ imply that $\bar{f}(x)=f(x)$ for every
$x\in\cY$. Moreover, for every $t\in\bR$,
\[
    \{x\in\bR^d:\bar{f}(x)\geq t\}
    =
    \begin{cases}
        \CH\bigl(\{z\in\cY:f(z)\geq t\}\bigr)
        & \text{if }t>m,\\
        \bR^d & \text{if }t\leq m,
    \end{cases}
\]
where $\CH(\emptyset)=\emptyset$.
Thus, $\bar{f}$ is quasi-concave and has the same maximum
as $f$.

For a quality function
$\cQ:\cU^*\times\cY\rightarrow\bR$, we apply this
extension to $\cQ(S,\cdot)$ for every fixed dataset $S$.

When running a solver on an enlarged domain, we map
any final output outside the original $\CH(\cY)$ to
a fixed public point in $\cY$.
This cannot decrease its quality and preserves privacy
by post-processing.
\end{remark}

With a slight abuse of notation, we use $f$ to denote both the original function and its canonical extension throughout the paper.

\subsection{1-dimensional algorithm}
We introduce Algorithm~\ref{alg:1-dim-opt-new}, a one-dimensional private QCO solver with an agnostic additive-error guarantee: it approximates the actual maximum without requiring its value or a lower bound as input. For a fixed sensitivity bound, the additive error is independent of the magnitude of the optimum. Our construction follows the core recursive approach of $\textsc{RecConcave}$~\cite{beimel2013private}, whose stated guarantee is multiplicative relative to a supplied quality promise. We believe our formulation also admits a simpler implementation.

At a high level, the algorithm identifies the scale of an interval on which the quality is close to optimal. It formulates scale selection as another QCO problem over a domain of size $O(\log|\cY|)$ and solves it recursively. Once this scale is identified, private selection procedures locate a point of high quality.

For simplicity, we assume that $|\cY|$ is a tower of $2$,
including $1$. Otherwise, for the original domain
$\cY=\{0,1,\cdots,R\}$, extend it to
$\widetilde{\cY}=\{0,1,\cdots,\widetilde{R}\}$,
where $\widetilde{R}+1$ is the smallest such tower
at least $R+1$. Evaluate $\cQ$ on the padded domain
using its canonical extension from
Remark~\ref{rem:quasi-func-ext}.
After the initial call returns $\widetilde{y}$,
return $\hat{y}=\min\{\widetilde{y},R\}$.
Also, $\log^*(\widetilde{R}+1)=\log^*(R+1)$.

All intervals below consist of integer points. For $I=[a,b]$ of size $J=b-a+1$, let $I^{\otimes 3}=[a-J,b+J]$. Complements are taken inside the current domain $\cY$, and we use the convention $\max\emptyset=-\infty$.

\begin{algorithm}
\caption{Private 1-dimensional Quasi-Concave Optimization}\label{alg:1-dim-opt-new}
\begin{algorithmic}[1]
\Require Parameters $\eps'>0,\beta',\delta'\in(0,1)$, an input domain $\cU$, an output domain $\cY = \{0,1,\cdots,R\}$ padded as above, quasi-concave quality function $\cQ:\cU^*\times \cY \rightarrow \bR$ with sensitivity 1, recursive height $H$ (initially set $H= \log^*(|\cY|)$), public stability parameter $\Delta>0$, and a dataset $S\in \cU^*$.

\State If $|\cY|=1$ return $0$.

    \State On the initial call only, set $\eps\leftarrow\eps'/(2H)$, $\beta\leftarrow\beta'/(2H)$ and $\delta\leftarrow\delta'/H$. Recursive calls use these same parameters without rescaling.


\State For $j \in \cY'=\{0,1\cdots,\lfloor\log(\max(\cY))\rfloor\}$, let $\cI_j=\{[a,b]: a,b\in \cY, b-a=2^j-1\}$. If $(I^{\otimes 3})^C=\emptyset$, the second term in the minimum below is $+\infty$, so only the first term is used.

\State Let $r=\max_{p\in \cY}\cQ(S,p)$, and define: 
\begin{align*}
    q(S,j)=\frac{1}{2}\max_{I\in \cI_j} \left\{ \min \left\{ \min_{p\in I}\cQ(S,p) - (r-\Delta), \quad r-\max_{p\in (I^{\otimes 3})^C}\cQ(S,p) \right\} \right\}
\end{align*}

\State Run the Algorithm~\ref{alg:1-dim-opt-new} recursively with parameters $\eps,\beta,\delta$, an input domain $\cU$, an output domain $\cY'$, quasi-concave quality function $q:\cU^*\times \cY' \rightarrow \bR$, recursive height $H-1$, stability parameter $\Delta/8$ , dataset $S\in \cU^*$, and let the output be $\hat{j}$. \label{alg:rec}

\State Let $J=2^{\hat{j}}$, $K=\frac{\max(\cY)+1}{J}$ and define $C_{J}=\left\{[0,J),[J,2J),\cdots  [(K-1)J,KJ)\right\}$. \label{line:m-cover}

\State Run the GAP-MAX algorithm (Theorem~\ref{thm:gap-max}) on $C_{J}$ with score $\score(S,I) = \max_{y\in I} \cQ(S,y)$ for $I\in C_{J}$ and let the output be $\hat{I}=\text{GAP-MAX}(C_J,S,\score,\Delta/40,\beta,\eps,\delta,4)$.

\State Run the Exponential Mechanism (Theorem~\ref{thm:exp-mech}) on the integer endpoints $\{\min\hat{I},\max\hat{I}\}$ of $\hat{I}$ with score function $\score(S,y) = \cQ(S,y)$, parameters $\eps,\beta$ and return the output.
\end{algorithmic}
\end{algorithm}

\begin{theorem}\label{thm:1-dim-opt-new}
Let $C>0$ be a sufficiently large constant,
$R\geq1$ be an integer, and
$\eps>0,\beta,\delta\in(0,1)$ be parameters.
Let $\cU$ be an arbitrary domain, and let
$\cY=\{0,1,\cdots,R\}$\footnote{
Without loss of generality, the algorithm applies to any finite
ordered domain $\{y_0<y_1<\cdots<y_R\}$.
We identify each $y_j$ with $j$, define the mapped quality function
$\cQ'(S,j)=\cQ(S,y_j)$, and map the output $\hat{j}$ back to
$y_{\hat{j}}$. This order-preserving mapping preserves
quasi-concavity and sensitivity.
Thus, throughout the paper, we may invoke the algorithm on any
finite ordered domain, setting $R$ to one less than its size.
} be a finite domain.
Furthermore, let $\cQ:\cU^*\times\cY\rightarrow\bR$
be quasi-concave in its second argument for every dataset,
with sensitivity $1$ in its first argument.
Then Algorithm~\ref{alg:1-dim-opt-new} is $(\eps,\delta)$-DP and for any
\[
    \Delta\geq C\, 8^{\log^{*}(R+1)}
    \left(
        1+\frac{\log^{*}(R+1)}{\eps}
        \log\left(
            \frac{8\left(\log^{*}(R+1)\right)^2}{\beta\delta}
        \right)
    \right),
\]
and every dataset $S\in\cU^*$, its output
$\hat{y}\in\cY$ satisfies
\begin{align*}
    \cQ(S,\hat{y})
    \geq \max_{y\in\cY}\cQ(S,y)-\Delta
\end{align*}
with probability at least $1-\beta$.
\end{theorem}

\begin{proof}
 By Remark~\ref{rem:quasi-func-ext}, padding preserves
quasi-concavity, sensitivity $1$, and the maximum quality.
The final clamping cannot decrease the output's quality
and preserves privacy by post-processing.
Since padding also preserves $\log^*(R+1)$,
we may assume that $|\cY|=R+1$ is a tower of $2$.

Throughout the induction below, $\eps,\beta,\delta$
denote the parameters after the initial rescaling
in the algorithm.

    First, we show that for every dataset $S\in \cU^*$ and every finite range $\cY = \{0,1\cdots,R\}$ with $R\geq 1$, the function $q:\cU^*\times \cY' \rightarrow \bR$:

$$q(S,j)=\frac{1}{2}\max_{I\in \cI_j} \left\{ \min \left\{ \min_{p\in I}\cQ(S,p) - (r-\Delta), \quad r-\max_{p\in (I^{\otimes 3})^C}\cQ(S,p) \right\} \right\}$$
 is quasi-concave w.r.t. its second argument, where  $\cI_j=\{[a,b]: a,b\in \cY, b-a=2^j-1\}$ for $j \in \cY'=\{0,1\cdots,\lfloor\log(\max(\cY))\rfloor\}$, 
 $r=\max_{p\in \cY}\cQ(S,p)$, and $\Delta>0$ is an arbitrary parameter.

Let $j_1\leq j_2 \leq j_3$ be arbitrary elements of $\cY'$. We show that $q(S,j_2)\geq \min\{q(S,j_1),q(S,j_3)\}=\alpha$. The cases $j_2=j_1$ or $j_2=j_3$ are trivial, so we assume $j_1 < j_2 < j_3$. Let $I_1\in \cI_{j_1}$ and $I_3\in \cI_{j_3}$ be any intervals maximizing $q(S,j_1)$ and $q(S,j_3)$ respectively. 

Since the function $\cQ$ is quasi-concave, we can say its upper level sets are convex sets (intervals). So, let the interval $L_A=\{y\in\cY: \cQ(S,y)\geq r-\Delta+2\alpha\}$ to be its $(r-\Delta+2\alpha)$-level set.

Note that $q(S,j)\geq \alpha$ if and only if there exists an interval $I\in \cI_j$ such that $\min\limits_{p\in I}\cQ(S,p) \geq 2\alpha + r -\Delta$ and $\max\limits_{p\in (I^{\otimes 3})^C}\cQ(S,p)\leq r-2\alpha$.

So we get that $I_1,I_3\subseteq L_A$, which implies that $|L_A|\geq 2^{j_3}>2^{j_2}>2^{j_1}$. This means that we can extend the interval $I_1$ inside $L_A$ to an interval $I_2$ of length $2^{j_2}$. In other words, there exists $I_2$ of size $2^{j_2}$ satisfying $I_1\subseteq I_2 \subseteq L_A$. For any choice of such $I_2$, we have:
\begin{align}
\min_{p\in I_2}\cQ(S,p) \geq 2\alpha + r -\Delta. \label{eq:m1}
\end{align}

We know that $\max\limits_{p\in (I_1^{\otimes 3})^C}\cQ(S,p)\leq r-2\alpha$. Also, since $I_1\subseteq I_2$, we have $I_1^{\otimes 3} \subseteq I_2^{\otimes 3}$. Therefore, $(I_2^{\otimes 3})^C \subseteq (I_1^{\otimes 3})^C$, which results in:
\begin{align}
\max_{p\in (I_2^{\otimes 3})^C}\cQ(S,p) \leq \max_{p\in (I_1^{\otimes 3})^C}\cQ(S,p)\leq r-2\alpha    \label{eq:m2}
\end{align}

Putting \eqref{eq:m1} and \eqref{eq:m2} together, concludes the result $q(S,j_2)\geq \alpha$.

Also, note that the $q$ function has sensitivity 1 w.r.t. its first argument. This is because the function $\cQ$ has sensitivity 1, and $r=\max_{y\in \cY}\cQ(S,y)$ also has sensitivity 1, so both terms have sensitivity at most 2. The $\min$ and $\max$ operators do not increase the sensitivity, and the coefficient $1/2$ gives sensitivity 1. If the complement is empty, only the first term is used; this case is fixed by the domain and the interval, independently of $S$. In particular, $q$ is finite everywhere.

At recursive height $0\leq h \leq H=\log^*(R+1)$, we denote the output domain by $\cY_h$, the objective function by $q_h$, stability parameter $\Delta_h>0$, the maximum value by $r_h=\max_{y\in \cY_h}q_h(S,y)$, and the output by $y_h\in \cY_h$.

First, note that $|\cY_h|=\log^{\left\lfloor H-h \right\rfloor}(R+1)=\text{Tow}(h)$ (see Definition~\ref{def:iter-log-floor-n})
where;
\begin{align*}
    \text{Tow}(h)=\begin{cases}
1 & \text{if } h = 0 \\
2^{\text{Tow}(h-1)} & \text{if } h \geq 1.
\end{cases}
\end{align*}

Indeed, $|\cY_{h-1}|=\lfloor\log(|\cY_h|-1)\rfloor+1=\log(|\cY_h|)=\text{Tow}(h-1)$ for $h\geq 1$, since $|\cY_h|$ is a tower of $2$.

Also, by definition, it holds that $\Delta_h=\dfrac{\Delta}{8^{H-h}}$ .

At recursive height $h=0$, it holds that $|\cY_0| = \log^{\left\lfloor H \right\rfloor}(R+1)= \text{Tow}(0)=1$ and algorithm returns the single element in the domain $y_0\in \cY_0=\{0\}$, which trivially satisfies the induction base $q_0(S,y_0)> \max\limits_{y\in \cY_0}q_0(S,y) - \Delta_0 = q_0(S,y_0) -\Delta_0$, where $\Delta_0=8^{-H}\Delta$.

By induction assumption, at the recursive height $h> 0$ (line~\ref{alg:rec}), the algorithm will return some $y_{h-1}\in \cY_{h-1}=\left\{0,1,\cdots,\lfloor\log(\max(\cY_{h}))\rfloor\right\}$ satisfying:
\begin{align*}
    q_{h-1}(S,y_{h-1})\geq \max_{j\in \cY_{h-1}}q_{h-1}(S,j)- \Delta_{h-1}
    \end{align*}
    with probability at least $1-2(h-1)\beta$.
    
This means, there exists an interval $I_h\subseteq \cY_h$ of length $2^{y_{h-1}}$, such that:
\begin{align}
    \frac{1}{2}\min \left\{ \min_{p\in I_h}q_{h}(S,p) - (r_h-\Delta_h), \quad r_h-\max_{p\in (I_h^{\otimes 3})^C}q_{h}(S,p) \right\}\geq \max_{j\in \cY_{h-1}}q_{h-1}(S,j)- \Delta_{h-1} \label{eq:m3}
\end{align}

   Consider the level set $L_{r_h-\frac{\Delta_h}{2}}=\{y\in \cY_h: q_h(S,y)\geq r_h-\frac{\Delta_h}{2}\}$ which is an interval since the function $q_h$ is quasi-concave. Denote this interval by $I_h^*$, which has the length $T=\Big|L_{r_h-\frac{\Delta_h}{2}}\Big|$, and satisfies $\min\limits_{p\in I_h^*}q_{h}(S,p)\geq r_h-\frac{\Delta_h}{2}$ and $\max\limits_{p\in (I_h^*)^C}q_h(S,p) <r_h-\frac{\Delta_h}{2}$. Let $t=\min\{\lfloor \log(T) \rfloor,\lfloor\log(\max(\cY_h))\rfloor\}$. Then $t\in\cY_{h-1}$ and $2^t\leq T\leq 2^{t+1}$, including the case $I_h^*=\cY_h$, where $2^t=|\cY_h|/2$. Every sub-interval $J_h^*\subseteq I_h^*$ of size $2^t$ has $I_h^*\subseteq (J_h^*)^{\otimes 3}$, since the portions of $I_h^*$ on either side have size at most $T-2^t\leq 2^t$. Therefore, it satisfies $\min\limits_{p\in J_h^*}q_h(S,p)\geq r_h-\frac{\Delta_h}{2}$ and $\max\limits_{p\in ((J_h^*)^{\otimes 3})^C}q_h(S,p) <r_h-\frac{\Delta_h}{2}$.  Now, note that for any such $J_h^*$, we have:
    \begin{align*}
        \max_{j\in \cY_{h-1}}q_{h-1}(S,j) &\geq q_{h-1}(S,t)\\
        &\geq \frac{1}{2}\min \left\{ \min_{p\in J_h^*}q_{h}(S,p) - (r_h-\Delta_h), \quad r_h-\max_{p\in ((J_h^*)^{\otimes 3})^C}q_{h}(S,p) \right\} \\
        &\geq \frac{1}{2}\min \left\{ \frac{\Delta_h}{2}, \quad \frac{\Delta_h}{2} \right\} = \frac{\Delta_h}{4} \\
    \end{align*}

    Putting together with \eqref{eq:m3}, we get that:

    \begin{align}
   \frac{1}{2} \min \left\{ \min_{p\in I_h}q_{h}(S,p) - (r_h-\Delta_h), \quad r_h-\max_{p\in (I_h^{\otimes 3})^C}q_{h}(S,p) \right\}\geq \frac{\Delta_h}{4}- \Delta_{h-1} = \frac{\Delta_h}{8}
\end{align}
with probability at least $1-2(h-1)\beta$.

   This implies $\min\limits_{p\in I_h}q_{h}(S,p)\geq r_h-\frac{3\Delta_h}{4}$ and $r_h-\frac{\Delta_h}{4} \geq \max\limits_{p\in (I_h^{\otimes 3})^C}q_{h}(S,p)$. 

    Let $C_J$ be the cover defined as in line~\ref{line:m-cover}. Also, for every $I\in C_J$, define the $\score$ function to be $\score(S,I) = \max\limits_{i\in I} q_h(S,i)$. Now, we have:
    \begin{align*}
        &\left|\left\{I \in C_J: \score(S,I) \geq \sup_{I \in C_J} \score(S,I) - 5\left(\frac{\Delta_h}{40}\right) \right\}\right| \\
        =& \left|\left\{I \in C_J: \score(S,I) \geq r_h - \frac{\Delta_h}{8}\right\}\right|\leq 4
    \end{align*}
    This is because $I_h^{\otimes 3}$ can intersect with at most $4$ many partitions from $C_J$, and any point outside $I_h^{\otimes 3}$ has score no more than $r_h-\frac{\Delta_h}{4}$.
    Using Theorem~\ref{thm:gap-max}, we can say the output $\hat{I}=\text{GAP-MAX}(C_J,S,\score,\frac{\Delta_h}{40},\beta,\eps,\delta,4)$
    will satisfy $\score(S,\hat{I})\geq r_h-\frac{\Delta_h}{40}$ with conditional probability at least $1-\beta$, provided $\Delta_h\geq 40C_0\left(1+\frac{\log(4/(\delta\beta))}{\eps}\right)$ for a sufficiently large absolute constant $C_0\geq 1$. 
    
    Let the endpoints of $\hat{I}$ be $\{\hat{I}_s,\hat{I}_e\}=\{\min\hat{I},\max\hat{I}\}$. We will show that at least one of these endpoints has quality at least $r_h-\frac{3\Delta_h}{4}$. Let $i^*\in\argmax_{i\in\hat{I}}q_h(S,i)$, so that $q_h(S,i^*)\geq r_h-\frac{\Delta_h}{40}$.

    If $\hat{I}\cap I_h\neq\emptyset$, since both intervals have size $J$, at least one endpoint of $\hat{I}$ belongs to $I_h$, and thus has quality at least $r_h-\frac{3\Delta_h}{4}$.

    Otherwise, suppose that $\hat{I}$ lies to the left of $I_h$, and let $a=\min I_h$. Then $i^*\leq\hat{I}_e<a$. Since the function $q_h$ is quasi-concave, we have:
    \begin{align}
        q_h(S,\hat{I}_e) \geq \min\{q_h(S,i^*),q_h(S,a)\}\geq r_h-\frac{3\Delta_h}{4}.
    \end{align}
    The case where $\hat{I}$ lies to the right of $I_h$ is symmetric, using $\hat{I}_s$ and $\max I_h$. This also covers $J=1$, when the two endpoints coincide.
    
    After running the Exponential Mechanism (Theorem~\ref{thm:exp-mech}) on $\{\hat{I}_s, \hat{I}_e\}$, quality function $q_h$, with parameters $\eps,\beta$, the output $\hat{y}$ satisfies:
    \begin{align*}
        q_h(S,\hat{y}) &\geq \max_{y \in \{\hat{I}_s, \hat{I}_e\}} q_h(S,y) - \frac{2 \log(2/\beta)}{\eps} \geq r_h-\frac{3\Delta_h}{4} - \frac{2 \log(2/\beta)}{\eps} \geq r_h - \Delta_h
    \end{align*}
    with probability at least $1-2h\beta$, concluding the induction step. The last inequality requires $\Delta_h\geq 8\log(2/\beta)/\eps$, which follows from the condition above. The two new failure events are the GAP-MAX call and the Exponential Mechanism call, each with conditional probability at most $\beta$.

    Thus, at recursive height $H=\log^*(R+1)$, the output $\hat{y}$ of the algorithm satisfies:
    \begin{align*}
        q_H(S,\hat{y}) = \cQ(S,\hat{y}) \geq r_H - \Delta_H = \max_{y\in \cY} \cQ(S,y) - \Delta_H = \max_{y\in \cY}\cQ(S,y) - \Delta
    \end{align*}
    with probability at least $1-2H\beta$.

    We now, prove the privacy of the algorithm. For every fixed output of the preceding private calls, the interval scores and endpoint scores have sensitivity 1. At each recursion, the algorithm invokes GAP-MAX (Theorem~\ref{thm:gap-max}) which is $(\eps,\delta)$-DP, and Exponential Mechanism (Theorem~\ref{thm:exp-mech}) which is $(\eps,0)$-DP. Applying (basic) composition gives the total privacy cost of $(2\eps,\delta)$-DP at each recursive step.

    Now, we use (basic) composition over all the recursive calls of the algorithm. There are exactly $H=\log^*(R+1)$ nontrivial recursive levels, which leads to the total privacy cost of $\left(2H\eps ,H\delta\right)$.

    At each nontrivial recursion call of the algorithm, it is sufficient that $\Delta_h = \frac{\Delta}{8^{H-h}}\geq \frac{\Delta}{8^{H}}\geq 40C_0\left(1+\frac{\log(4/(\delta\beta))}{\eps}\right)$. Now, let $\eps',\beta',\delta'$ be the original input parameters. The initial rescaling sets $\beta=\frac{\beta'}{2H}$, $\eps = \frac{\eps'}{2H}$, $\delta=\frac{\delta'}{H}$, which gives the sufficient condition $\Delta \geq 40C_0 8^H\left(1+\frac{2H}{\eps'} \log\left(\frac{8H^2}{\beta'\delta'}\right)\right)$.
    Under this condition, the output satisfies $\cQ(S,\hat{y})\geq\max_{y\in\cY}\cQ(S,y)-\Delta$ with probability at least $1-\beta'$. Moreover, the algorithm is $(\eps',\delta')$-DP. Recall that, the domain padding preserves $H=\log^*(R+1)$. Choosing $C\geq 80C_0$ concludes the proof.
\end{proof}

\subsection{High-dimensional algorithm}
We now extend the one-dimensional solver developed in the previous subsection to high dimensions by optimizing one coordinate at a time. Given the coordinates already optimized, we score each candidate coordinate value by the maximum quality attainable over all completions of the remaining coordinates. A suitable public extension of the coordinate domains ensures that each induced one-dimensional quality function attains its maximum within its finite candidate domain. The additive errors accumulate across the $d$ coordinates, contributing a factor of $d$, while advanced composition contributes an additional $\sqrt{d}$ factor. Similar coordinatewise reductions were previously used for private center-point computation and learning halfspaces~\cite{pmlr-v99-beimel19a,kaplan2020private}.

After fixing the previously optimized coordinates, a maximizing completion of the remaining coordinates need not lie in the original finite grid. The following lemma constructs a public grid extension that contains a maximizing completion for every such prefix, while controlling the domain size.
\begin{lemma}[Domain pre-processing]
\label{lemma:coor-extension}
    Let $\cY=\cY_1\times\cdots\times\cY_d\subseteq\bR^d$
    be a non-empty finite grid domain.
    Then, there exists a finite domain $\cY'=\cY_1'\times\cdots\times\cY_d'$, satisfying:
    \begin{quote}
        \textit{(1)} $\cY_i\subseteq\cY_i'\subseteq\left[\min\cY_i,\max\cY_i\right]$ for every $i\in[d]$.
        Consequently,
        $\cY\subseteq\cY'\subseteq\CH(\cY)$.

        \textit{(2)} For any non-empty $\cA\subseteq\cY$, and any prefix $(y_1,\ldots,y_i)\in\cY_1'\times\cdots\times\cY_i'$,
        if there exists a feasible extension $(y_1,\ldots,y_i, y_{i+1}',\dots,y_d')\in\CH(\cA)$, then there exists a feasible extension $(y_1,\ldots,y_i, y_{i+1}'',\dots,y_d'')\in\CH(\cA)\cap\cY'$ with the same prefix.

        \textit{(3)} For every $i\in[d]$, it holds that
        $|\cY_i'|\leq2^{|\cY|2^{i-1}}$.
        In particular, $|\cY'|\leq2^{|\cY|(2^d-1)}$, and
        $\log^{*}|\cY'|\leq\log^{*}|\cY|+3$.
    \end{quote}
\end{lemma}

\begin{proof}
    We construct $\cY_i'$'s inductively. First, set $\cY_1'=\cY_1$.
    To construct $\cY_i'$ for $i\in\{2,\ldots,d\}$,
    start with $\cY_i'=\emptyset$.
    For every tuple
    $(y_1',\ldots,y_{i-1}')\in
    \cY_1'\times\cdots\times\cY_{i-1}'$
    and every non-empty $\cA\subseteq\cY$, include
    \[
        \min_{\substack{(y_1',\ldots,y_{i-1}',y_i,y_{i+1},\ldots,y_d)\in\CH(\cA)}}y_i
    \]
    in $\cY_i'$ whenever the feasible set is non-empty.
    Every non-empty feasible set is compact, so its minimum
    is attained.

    To see property (1), fix any $y=(y_1,\ldots,y_d)\in\cY$.
    By definition, $y_1\in\cY_1'=\cY_1$.
    For the induction step, assume
    $(y_1,\ldots,y_{i-1})\in\cY_1'\times\cdots\times\cY_{i-1}'$.
    Taking this prefix and $\cA=\{y\}$ includes $y_i$ in $\cY_i'$.
    Hence, $\cY_i\subseteq\cY_i'$ for every $i$.
    Also, every value added to $\cY_i'$ is a coordinate of a point
    in $\CH(\cY)$, and therefore lies in
    $[\min\cY_i,\max\cY_i]$.
    Since $\cY$ is a grid domain, we have
    $\CH(\cY)=\prod_{j=1}^d[\min\cY_j,\max\cY_j]$,
    which concludes $\cY'\subseteq\CH(\cY)$.

    To see property (2), fix any non-empty $\cA\subseteq\cY$
    and a feasible prefix
    $(y_1,\ldots,y_i)\in\cY_1'\times\cdots\times\cY_i'$.
    At iteration $j>i$, the minimum feasible value of the
    $j$-th coordinate in $\CH(\cA)$, subject to the coordinates
    already fixed, belongs to $\cY_j'$ by construction.
    Since the minimum is attained, fixing this coordinate
    leaves a non-empty feasible set.
    Continuing to the last coordinate gives a point
    in $\CH(\cA)\cap\cY'$ with the fixed prefix.

    Finally, let $N=|\cY|$.
    In the first iteration, $|\cY_1'|=|\cY_1|\leq N\leq2^N$.
    At every step $i\geq2$, at most one value is added per
    prefix and non-empty $\cA\subseteq\cY$.
    Thus, induction gives
    \[
        |\cY_i'|
        \leq2^N\prod_{j<i}|\cY_j'|
        \leq2^{N2^{i-1}}.
    \]
    As a result,
    $|\cY'|
        =\prod_{i=1}^d|\cY_i'|
        \leq2^{N(2^d-1)}$. To obtain the bound on $\log^{*}|\cY'|$, let
    $r=|\{i\in[d]:|\cY_i|>1\}|\leq d$.
    By property (1), every singleton coordinate domain remains
    a singleton. Applying the same recurrence only to the
    $r$ non-singleton coordinate domains gives
    $|\cY'|\leq2^{N(2^r-1)}
        \leq2^{N(N-1)}$,
    where the second inequality follows from
    $2^r\leq\prod_{i=1}^d|\cY_i|=N$.
    If $N=1$, then $|\cY'|=1$ and the claim is immediate.
    Otherwise, $|\cY'|\leq2^{2^{2N}}$, which implies
    $\log^{*}|\cY'|
        \leq2+\log^{*}(2N)
        \leq\log^{*}N+3$.
\end{proof}

\begin{algorithm}
\caption{High-dimensional Quasi-Concave Optimization}
\label{alg:high-dim-quasi}
\begin{algorithmic}[1]
\Require Parameters $\beta,\delta,\eps\in(0,1)$,
an input domain $\cU$, a public finite domain
$\cY\subseteq\bR^d$ with $|\cY|\geq2$,
a dataset $S\in\cU^*$, and a discrete quasi-concave
quality function $\cQ:\cU^*\times\cY\rightarrow\bR$
with sensitivity $1$.

\If{$|\cY|\leq d+1$}
    \State \Return The output of Exponential Mechanism
    (Theorem~\ref{thm:exp-mech}) with quality
    function $\cQ$, dataset $S$, candidate domain
    $\cY$, and privacy parameter $\eps$.
\EndIf

\State Use the canonical extension from
Remark~\ref{rem:quasi-func-ext} when evaluating $\cQ$ outside of $\cY$.

\State For each $i\in[d]$, let
$\mathcal G_i=\{y_i:(y_1,\ldots,y_d)\in\cY\}$,
and set
$\mathcal G=\mathcal G_1\times\cdots\times\mathcal G_d$. \Comment{Extend $\cY$ to a Cartesian grid.}

\State Preprocess $\mathcal G$ to obtain the coordinate
extension $\cY'=\cY_1'\times\cdots\times\cY_d'$
using Lemma~\ref{lemma:coor-extension}.
\Comment{This step is data-independent.}

\State Set
$\Delta_0=C d^{1.5}
\frac{8^{\log^{*}|\cY|}\log^{*}|\cY|
\sqrt{2\log(2/\delta)}}{\eps}
\log\left(
\frac{16\left(d\log^{*}|\cY|\right)^2}{\beta\delta}
\right)$,
where $C>0$ is a sufficiently large constant.

\State Set
$\beta'=\frac{\beta}{d}$,
$\eps'=\frac{\eps}{2\sqrt{2d\log(2/\delta)}}$, and
$\delta'=\frac{\delta}{2d}$.

\For{$i\in\{1,2,\ldots,d\}$}

    \State Define
    $\cQ_i:\cU^*\times\cY_i'\rightarrow\bR$ as
    \begin{align*}
        \cQ_i(S,y):=
        \max_{(y_{i+1}^*,\ldots,y_d^*)\in\bR^{d-i}}
        \cQ\bigl(S,(\hat{y}_1,\ldots,\hat{y}_{i-1},y,
        y_{i+1}^*,\ldots,y_d^*)\bigr).
    \end{align*}

        \State Run the QCO Algorithm~\ref{alg:1-dim-opt-new}
        with parameters
        $\beta',\eps',\delta',\Delta=\frac{\Delta_0}{d}$,
        quality function $\cQ_i$,
        dataset $S$, input domain $\cU$, and output domain
        $\cY_i'$. Let $\hat{y}$ be its output, and set $\hat{y}_i\leftarrow\hat{y}$.

\EndFor

\State Set $\hat y=[\hat y_1,\ldots,\hat y_d]$.
\If{$\hat y\notin\CH(\cY)$}
    \State \Return a fixed public point of $\cY$.
\EndIf
\State \Return $\hat y$.
\end{algorithmic}
\end{algorithm}

The next theorem establishes the privacy and agnostic additive-error guarantees of the following high-dimensional QCO algorithm.
\begin{theorem}\label{thm:high-dim-opt}
Let $\beta,\eps,\delta\in(0,1)$ be parameters,
$C>0$ be a sufficiently large constant, and let
$\cY\subseteq\bR^d$
be a public finite domain with $|\cY|\geq2$.
Furthermore, let
$\cQ:\cU^*\times\cY\rightarrow\bR$
be a discrete quasi-concave function with sensitivity $1$.
Then Algorithm~\ref{alg:high-dim-quasi} is $(\eps,\delta)$-DP.
Furthermore, for any
\[
    \Delta\geq C d^{1.5}
    \frac{8^{\log^{*}|\cY|}\log^{*}|\cY|
    \sqrt{2\log(2/\delta)}}{\eps}
    \log\left(
    \frac{16\left(d\log^{*}|\cY|\right)^2}{\beta\delta}
    \right),
\]
and every dataset $S\in\cU^*$, its output
$\hat{y}\in\CH(\cY)$ satisfies
\begin{align*}
    \cQ(S,\hat{y})
    \geq \max_{y\in\cY}\cQ(S,y)-\Delta
\end{align*}
with probability at least $1-\beta$.
\end{theorem}

\begin{proof}
First suppose that $2\leq|\cY|\leq d+1$.
In this case, the algorithm runs the Exponential
Mechanism, which is $(\eps,0)$-DP and, with
probability at least $1-\beta$, returns $\hat y$
satisfying
\[
    \cQ(S,\hat y)
    \geq
    \max_{y\in\cY}\cQ(S,y)
    -
    \frac{2\log((d+1)/\beta)}{\eps}.
\]
For sufficiently large $C$, this error is at most
the stated $\Delta$. Thus, both the privacy and
utility guarantees hold in this case. We next, assume that $|\cY|>d+1$.

    Let $\mathcal G$ be the Cartesian grid constructed
by the algorithm. Since $d<|\cY|$, we have $|\mathcal G|
    \leq|\cY|^d
    \leq|\cY|^{|\cY|}
    \leq2^{2^{|\cY|}}$,
and therefore $\log^*|\mathcal G|\leq\log^*|\cY|+2$.

Let $\cY'=\cY_1'\times\cdots\times\cY_d'$
be the coordinate extension of $\mathcal G$
from Lemma~\ref{lemma:coor-extension}.
In particular, $\cY\subseteq\mathcal G\subseteq\cY'$.
Throughout the proof, $\cQ$ denotes its canonical
extension from the original domain $\cY$.

    Fix an iteration $i\in[d]$ and previously chosen coordinates
    $(\hat{y}_1,\ldots,\hat{y}_{i-1})\in
    \cY_1'\times\cdots\times\cY_{i-1}'$.
    For every $y\in\cY_i'$, the canonical extension
takes values in the finite set
$\{\cQ(S,z):z\in\cY\}$.
Thus, $\cQ_i(S,y)$ is finite and its maximum
over completions is attained.
For $i=d$, the definition gives
\[
    \cQ_d(S,y)
    =
    \cQ(S,(\hat y_1,\ldots,\hat y_{d-1},y)).
\]

We next show that a maximizing completion exists
in $\cY'$.
If
$\cQ_i(S,y)=\min_{z\in\cY}\cQ(S,z)$,
every completion in $\cY'$ with the specified first
$i$ coordinates attains this value: the canonical
extension is bounded below by this minimum, and
no completion with this prefix can exceed
$\cQ_i(S,y)$.

Otherwise, let $\cA=
    \{z\in\cY:\cQ(S,z)\geq\cQ_i(S,y)\}$.
Since $\cQ_i(S,y)>\min_{z\in\cY}\cQ(S,z)$,
a maximizing completion belongs to $\CH(\cA)$
by Remark~\ref{rem:quasi-func-ext}.
Since $\cA\subseteq\cY\subseteq\mathcal G$,
property (2) of Lemma~\ref{lemma:coor-extension},
applied to $\mathcal G$, gives a completion in
$\CH(\cA)\cap\cY'$ with the same first $i$
coordinates. Its quality is at least
$\cQ_i(S,y)$, and hence equals $\cQ_i(S,y)$.
Thus, in both cases,
    \begin{equation}\label{eq:coordinate-grid-max}
        \cQ_i(S,y)=
        \max_{(y_{i+1}^*,\ldots,y_d^*)\in
        \cY_{i+1}'\times\cdots\times\cY_d'}
        \cQ\bigl(S,(\hat{y}_1,\ldots,\hat{y}_{i-1},y,
        y_{i+1}^*,\ldots,y_d^*)\bigr).
    \end{equation}
    For $i=d$, the product over the remaining coordinates
    consists of the empty tuple.

    Now, we prove that $\cQ_i(S,\cdot)$ is discretely quasi-concave.
    Let $a,b,c\in\cY_i'$ and suppose that
    $c=\lambda a+(1-\lambda)b$ for some $\lambda\in[0,1]$.
    Let $y_a,y_b\in\cY'$ be points attaining
    $\cQ_i(S,a)$ and $\cQ_i(S,b)$, respectively, with
    the fixed first $i-1$ coordinates.
    The first $i$ coordinates of $\lambda y_a+(1-\lambda)y_b$
    are $[\hat{y}_1,\ldots,\hat{y}_{i-1},c]$.
    By quasi-concavity of the canonical extension,
    \begin{align*}
        \cQ_i(S,c)
        &\geq\cQ(S,\lambda y_a+(1-\lambda)y_b)\\
        &\geq\min\{\cQ(S,y_a),\cQ(S,y_b)\}\\
        &=\min\{\cQ_i(S,a),\cQ_i(S,b)\}.
    \end{align*}

    Next, we prove that $\cQ_i$ has sensitivity $1$.
    Suppose that $S,S'$ are neighboring datasets.
    For every fixed non-empty $\cA\subseteq\cY$, sensitivity
    of $\cQ$ on $\cY$ gives
    \[
        \left|
        \min_{z\in\cA}\cQ(S,z)
        -\min_{z\in\cA}\cQ(S',z)
        \right|\leq1.
    \]
    Taking the maximum over the subsets whose convex hulls
contain a fixed $x\in\CH(\cY)$ shows that the canonical
extension also satisfies
$|\cQ(S,x)-\cQ(S',x)|\leq1$.
Outside $\CH(\cY)$, the extension equals
$\min_{z\in\cY}\cQ(S,z)$, which also has sensitivity
at most $1$. Thus, this sensitivity bound holds
for every fixed $x\in\bR^d$.
    By Equation~\eqref{eq:coordinate-grid-max}, $\cQ_i$ is a
    maximum over a fixed finite set of such points.
    Consequently, for every fixed preceding transcript,
    $|\cQ_i(S,y)-\cQ_i(S',y)|\leq1$.

    Now, we are ready to use the 1-dimensional QCO solver from
    Algorithm~\ref{alg:1-dim-opt-new} on each $\cQ_i$.
    By Lemma~\ref{lemma:coor-extension}, $\log^*|\cY_i'|
    \leq\log^*|\mathcal G|+3
    \leq\log^*|\cY|+5$.
Singleton coordinate domains are returned
deterministically with zero error. For the remaining
coordinates, we apply Theorem~\ref{thm:1-dim-opt-new}.
The additive constant $5$ is absorbed by taking
the universal constant $C$ sufficiently large.
    Substituting, our parameter choices 
    $\frac{1}{\eps'}=
        \frac{2\sqrt{2d\log(2/\delta)}}{\eps}$, and 
        $\frac{1}{\beta'\delta'}=
        \frac{2d^2}{\beta\delta}$ into
    Theorem~\ref{thm:1-dim-opt-new}, for sufficiently large $C$,
    the choice $\Delta_0/d$ satisfies its error requirement.
    Thus, conditional on every fixed preceding transcript,
    the output $\hat{y}_i$ in iteration $i$ satisfies:
    \[
        \cQ_i(S,\hat{y}_i)
        \geq\max_{y\in\cY_i'}\cQ_i(S,y)-\frac{\Delta_0}{d}
    \]
    with probability at least $1-\beta'$.

    We prove the utility guarantee on this event by induction.
    In the first iteration, Equation~\eqref{eq:coordinate-grid-max}
    gives
    \[
        \max_{y\in\cY_1'}\cQ_1(S,y)
        =\max_{y\in\cY'}\cQ(S,y)
        =\max_{y\in\cY}\cQ(S,y).
    \]
    The last equality follows because the canonical extension
    never exceeds the maximum quality on $\cY$, and
    $\cY\subseteq\cY'$.
    Hence,
    \[
        \cQ_1(S,\hat{y}_1)
        \geq\max_{y\in\cY}\cQ(S,y)-\frac{\Delta_0}{d}.
    \]

    For the induction step, let $i\in\{2,\ldots,d\}$.
    Applying Equation~\eqref{eq:coordinate-grid-max} to
    iterations $i-1$ and $i$ gives
    \[
        \max_{y\in\cY_i'}\cQ_i(S,y)
        =\cQ_{i-1}(S,\hat{y}_{i-1}),
    \]
    since both sides maximize $\cQ$ over all completions in
    $\cY_i'\times\cdots\times\cY_d'$ with the same
    first $i-1$ coordinates.
    Therefore, by the induction hypothesis,
    \begin{align*}
        \cQ_i(S,\hat{y}_i)
        &\geq\max_{y\in\cY_i'}\cQ_i(S,y)-\frac{\Delta_0}{d}\\
        &=\cQ_{i-1}(S,\hat{y}_{i-1})-\frac{\Delta_0}{d}\\
        &\geq\max_{y\in\cY}\cQ(S,y)-\frac{i}{d}\Delta_0.
    \end{align*}

    Let $\hat{y}=[\hat{y}_1,\ldots,\hat{y}_d]\in\cY'$
be the point produced by the coordinate loop,
before final postprocessing.
    Using a union bound,
    with probability at least $1-d\beta'=1-\beta$,
    \[
        \cQ(S,\hat{y})
        =\cQ_d(S,\hat{y}_d)
        \geq\max_{y\in\cY}\cQ(S,y)-\Delta_0
        \geq\max_{y\in\cY}\cQ(S,y)-\Delta.
    \]
    If $\hat y\notin\CH(\cY)$, its canonical quality
equals $\min_{z\in\cY}\cQ(S,z)$.
Replacing it by a fixed public point of $\cY$
therefore cannot decrease its quality.
Consequently, the final output belongs to
$\CH(\cY)$ and satisfies the same utility bound.
This final postprocessing also preserves privacy.

    The privacy guarantee follows from the privacy of the
    1-dimensional QCO algorithm
    (Theorem~\ref{thm:1-dim-opt-new}).
    Recall that, for every fixed preceding transcript, the quality
    function $\cQ_i$ has sensitivity $1$.
    Thus, each iteration is $(\eps',\delta')$-DP.
    Using advanced composition
    (Theorem~\ref{thm:advanced-composition}) with additional failure
    parameter $\delta/2$, after $d$ iterations the algorithm is
    $\left(
        \eps'\sqrt{2d\log(2/\delta)}
        +d\eps'(e^{\eps'}-1),\;
        d\delta'+\frac{\delta}{2}
        \right)$-DP.
    We have $d\delta'+\delta/2=\delta$.
    Moreover, $\eps'<1$ and $e^u-1\leq2u$ for $u\in[0,1]$, so
    \[
\eps'\sqrt{2d\log(2/\delta)}+d\eps'(e^{\eps'}-1)
        \leq\frac{\eps}{2}+2d(\eps')^2
        =\frac{\eps}{2}+\frac{\eps^2}{4\log(2/\delta)}
        <\eps
    \]
    using $\eps<1$ and $\log(2/\delta)>\log2>1/2$.
    Thus, the algorithm is $(\eps,\delta)$-DP.
\end{proof}

\end{document}